\documentclass[11pt]{article}
\usepackage[margin=1in]{geometry}
\usepackage[utf8]{inputenc}

\usepackage[english]{babel}
\usepackage[T1]{fontenc}
\usepackage{amssymb,amsmath,amstext,amsfonts,amsthm,braket}
\usepackage{mathtools}
\usepackage{verbatim}
\usepackage{relsize}
\usepackage[colorlinks=true, urlcolor=blue, linkcolor=blue, citecolor=blue]{hyperref}
\usepackage{fancyvrb}
\usepackage{tikz-cd}
\usepackage[percent]{overpic}
\usepackage{comment}

\usepackage{algorithm, algorithmic,enumerate,booktabs}

\newcommand{\algorithmicbreak}{\textbf{break}}
\newcommand{\BREAK}{\STATE \algorithmicbreak}

\usepackage{todonotes}
\usepackage{xcolor}
\usepackage{cleveref}

\usepackage[normalem]{ulem}
\usepackage{blkarray}
\usepackage{subcaption}
\usepackage{enumitem}
\usepackage[authoryear]{natbib}

\theoremstyle{plain}
\newtheorem*{theorem*}{Theorem}
\newtheorem{theorem}{Theorem}
\numberwithin{theorem}{section}
\newtheorem{proposition}[theorem]{Proposition}
\newtheorem{lemma}[theorem]{Lemma}
\newtheorem{corollary}[theorem]{Corollary}

\theoremstyle{definition}
\newtheorem{definition}[theorem]{Definition}
\newtheorem{remark}[theorem]{Remark}
\newtheorem{example}[theorem]{Example}

\newtheorem{assumption}[theorem]{Assumption}

\newcommand{\pa}{\text{pa}}

\newcommand{\trek}{\text{trek}}

\newcommand{\wh}{\text{wh}}

\definecolor{MyBlue}{RGB}{0,101,189} 
\definecolor{MyRed}{RGB}{234, 114, 55} 
\definecolor{MyGreen}{RGB}{162,173,0}

\newcommand{\id}{\text{id}}
\newcommand{\rem}{\text{rem}}
\newcommand{\deh}{\text{deh}}

\newcommand{\footremember}[2]{
    \footnote{#2}
    \newcounter{#1}
    \setcounter{#1}{\value{footnote}}
}

\title{Efficient Symbolic Computations for Identifying Causal Effects}

\author{
Benjamin Hollering\footremember{ben}{Max Planck Institute for Mathematics in the Sciences, Germany; \href{mailto:ben.hollering@tum.de}{benjamin.hollering@mis.mpg.de}}  \and 
Pratik Misra\footremember{pratik}{Binghamton University, USA; \href{mailto:pratik.misra@tum.de}{pmisra@binghamton.edu}} \and
Nils Sturma\footremember{nils}{École Polytechnique Fédérale de Lausanne, Switzerland; \href{mailto:nils.sturma@epfl.ch}{nils.sturma@epfl.ch}}}

\date{ }

\begin{document}
\maketitle
\begin{abstract}
Determining identifiability of causal effects from observational data under latent confounding is a central challenge in causal inference. For linear structural causal models, identifiability of causal effects is decidable through symbolic computation. However, standard approaches based on Gröbner bases become computationally infeasible beyond small settings due to their doubly exponential complexity. In this work, we study how to practically use symbolic computation for deciding rational identifiability. In particular, we present an efficient algorithm that provably finds the lowest degree identifying formulas. For a causal effect of interest, if there exists an identification formula of a prespecified maximal degree, our algorithm returns such a formula in  quasi-polynomial time.
\end{abstract}

\section{Introduction}
Identifiability of causal effects refers to studying whether it is feasible to infer cause-effect relationships under clearly detailed assumptions about the data-generating process. Determining whether causal effects are identifiable is crucial for any downstream task, such as robust estimation of effects or generalization across environments. The main challenge in identifying causal relations is the ubiquitous presence of latent (unobserved) variables, so we only observe marginal distributions. One of the most popular tools that allow us to argue about causal relationships are structural causal models \citep{spirtes2000causation, pearl2009causality}. Applied sciences widely make use of \emph{linear} structural causal models, which are valued for their simple interpretation \citep{bollen1989structural}. Linear structural causal models were introduced by \citet{wright1921correlation, wright1934method} and are also referred to as path diagrams. The precise setting is described as follows. Let $X=(X_v)_{v \in V}$ be a random vector that is indexed by a finite set $V$. Each linear causal model is defined  by a directed graph $G=(V, D, B)$, where the nodes $V$ correspond to the random variables, the directed edges $D \subseteq V \times V$ encode the causal relationships, and the bidirected edges $B \subseteq V \times V$ encode latent confounding. Note that $(v,w) \in B$ if and only $(w,v) \in B$. We suppose that all variables are related by noisy linear equations, that is, 
\begin{equation} \label{eq:structural-equation}
    X_v = \sum_{w \in \pa(v)} \lambda_{wv} X_w + \varepsilon_v,
\end{equation}
where $\pa(v)$ is the of parents of $v$ in the graph $G$, and $\varepsilon_v$ are stochastic noise variables with mean zero and finite variance. The distributional assumption is that, for $v \neq w$, the variables $\varepsilon_v$ and $\varepsilon_w$ are independent whenever there is no bidirected edge between $v$ and $w$. Now, the question of identification of a \emph{direct causal effect} $\lambda_{wv}$ refers to whether it is possible to recover $\lambda_{wv}$ from the covariance matrix of the observable random vector $X$. Writing the structural equations~\eqref{eq:structural-equation} in vector form, we get that
\[
    X = \Lambda^{\top} X + \varepsilon,
\]
where $\varepsilon = (\varepsilon_v)_{v \in V}$ is the noise vector and $\Lambda$ is the matrix of direct causal effects, that is, $\Lambda_{wv} = \lambda_{wv}$ whenever there is a directed edge from $w$ to $v$ in $G$ and  $\Lambda_{wv}=0$ else. By solving for $X$, we find that $X = (I - \Lambda)^{-\top} \varepsilon$ and therefore the covariance matrix is given by
\begin{equation} \label{eq:cov}
    \Sigma := \text{Var}[X] = (I - \Lambda)^{-\top} \Omega (I - \Lambda)^{-1},
\end{equation}
where $\Omega =(\omega_{vw}) = \text{Var}[\varepsilon]$. If it is possible to recover $\Lambda$ from the covariance matrix $\Sigma=(\sigma_{vw})$ via rational formulas in the entries of $\Sigma$, then the graph $G$ is said to be \emph{rationally identifiable}; we refer to Section~\ref{sec:algebraic-setup} for a precise definition.

\begin{example}
Consider a conditionally randomized study designed to evaluate the effect of a treatment on an outcome $Y$. Let $T$ denote the assigned treatment dose, which is randomized conditional on a baseline covariate  $L$. Since not all individuals adhere to their assigned treatment, we introduce $A$, representing the treatment actually received; see \citet[Section 16]{hernan2025causal} for background on adherence. The primary interest then is in the direct causal effect of the treatment actually taken $A$ on the outcome $Y$. In practice, an investigator may wish to account for latent confounding that partially explains the association between $A$ and $Y$. The corresponding data-generating process of this study can be represented by the mixed graph in Figure~\ref{fig:intro-example}.
\end{example}

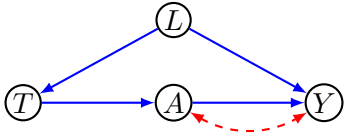
\begin{figure}[t]
\centering
\tikzset{
  every node/.style={circle, inner sep=0.3mm, minimum size=0.45cm, draw, thick, black, fill=white, text=black},
  every path/.style={thick}
}
\begin{tikzpicture}[align=center]
 \node[] (0) at (-2,1.1) {$L$};
  \node[] (1) at (-4,0) {$T$};
  \node[] (2) at (-2,0) {$A$};
  \node[] (3) at (-0,0) {$Y$};
  
  \draw[blue] [-latex] (0) edge (1);
  \draw[blue] [-latex] (0) edge (3);
  \draw[blue] [-latex] (2) edge (3);
   \draw[red, dashed, bend right] [latex-latex] (2) edge (3);
   \draw[blue] [-latex] (1) edge (2);
\end{tikzpicture}    
\caption{Mixed graph corresponding to a conditionally randomized trial with imperfect adherence.}
\label{fig:intro-example}  
\end{figure}

It was noted  by \citet{garcia2010identifying} that rational identifiability of a given graph is always decidable by computational algebraic geometry involving Gröbner basis computations. Indeed, if the graph is acyclic, each entry of the covariance matrix $\Sigma$ is a polynomial in the $\lambda_{vw}$ and $\omega_{ij}$ variables and thus deciding rational identifiability is identical to solving the polynomial system given in~\eqref{eq:cov}. However, the complexity of standard Gröbner basis methods can in the worst case be double exponential in the size of the graph \citep{mayr1997somcecomplexity, cox2015ideals}. Hence, they become infeasible even for small graphs on $5$ nodes. Recently, \citet{doerfler2024complexity} showed that rational identifiability can be decided in exponential running time. However, their algorithm is of a theoretical nature and it does not return the formulas for identification, which are of crucial interest in practice.  It is an open problem to \emph{graphically} characterize which graphs are rationally identifiable.  \\

Nevertheless, many graphical criteria have been given that are \emph{sufficient conditions} for rational identifiability. Crucially, they are efficient in the sense that they can be checked in polynomial time. They search for patterns in the covariance matrix and rely on maximum flow computations \citep{sullivant2010trek, cormen2009introduction}. The first sufficient condition was the instrumental variable criterion, see \citet{wright1928tariff} and \citet{bowden1984instrumental}. A major break through was the half trek criterion by \citet{foygel2012halftrek}. Improvements and follow-up works include criteria based on auxiliary variables (\citeauthor{chen2016incorporating}, \citeyear{chen2016incorporating} and \citeauthor{chen2017identification}, \citeyear{chen2017identification}), decomposition techniques \citep{tian2005identifying} and several generalizations and further developments, cf. \citet{tian2009parameter}, \citet{drton2016generic}, \citet{weihs2017determinantal}, \citet{kumor2019efficient}. To our knowledge, the most recent polynomial time sufficient condition for rational identifiability is the auxiliary cutset criterion by \citet{kumor2020efficient}. \\

In this paper, we take a different approach and remind ourselves that deciding rational identifiability corresponds to solving polynomial systems. Building on \citet{garcia2010identifying}, we study how symbolic computations based on Gröbner bases can be made practical for deciding rational identifiability. Roughly speaking, we apply two key ideas: First, we work with homogenized equations: A polynomial is homogeneous if all of its components are of the same degree. Second, we work degree by degree. That is, we always search for formulas with the lowest degree possible to identify a parameter. If no identifying formula is found up to the current degree bound, we increase the bound by one and repeat our computations. These two key ideas allow us to exploit the fact that degree-bounded Gröbner basis computations of homogeneous ideals are polynomial time. Moreover, to keep the degree during the computations as low as possible, we allow formulas that identify a new parameter to depend on the already identified parameters. \citet{abbott2017implicitization} considered a similar approach for efficiently computing implicit descriptions of hypersurfaces, a well-known challenge in the computational algebra community. \\

If a graph $G$ is rationally identifiable, then we denote by $\text{ID}_G$ the maximal degree of the polynomials appearing in the identifying formulas, when choosing each identifying formula with the lowest degree possible. Our main algorithm takes as input a mixed graph and a maximal degree $d$.  A shortened version of our main result then reads as follows.

\begin{theorem} \label{thm:main-result-intro}
Let $G=(V,D,B)$ be an acyclic mixed graph. If $G$ is rationally identifiable  and $d \geq \text{ID}_G$,  then  Algorithm~\ref{alg:degree-bounded-identification}  returns ``yes''. If $G$ is not rationally identifiable, then Algorithm~\ref{alg:degree-bounded-identification} returns ``no'' for all input degrees $d \in \mathbb{N}$. The computational complexity of Algorithm~\ref{alg:degree-bounded-identification} is of order 
\[
O\left( d |V|^5 \{3d|V|\}^{4\alpha d^2\log(3d|V|)}\right),
\]
where $\alpha$ is a constant such that row reduction of a $n \times n$ matrix can be performed in $O(n^{\alpha})$ operations. 
\end{theorem}

Said differently, if there exists an identification formula up to a pre-specified degree $d$, our algorithm finds it in quasi-polynomial time. Note that our algorithm also returns the symbolic formulas for identifying the direct causal effects. Moreover, our result implies that any identification method restricted to searching for formulas of bounded degree is subsumed by our approach: whenever such a formula exists, our method will also find it in quasi-polynomial time. \\

\begin{example}
We return to the earlier example of a conditionally randomized study with imperfect adherence, represented by the graph in Figure~\ref{fig:intro-example}. For this graph, all direct causal effects in the linear model~\eqref{eq:structural-equation} are rationally identifiable. If we apply the standard Gröbner basis methods as in \citet{garcia2010identifying}, we find the identification formulas
\[
    \lambda_{LT} = \frac{\sigma_{LT}}{\sigma_{LL}}, \quad 
    \lambda_{LY} = \frac{\sigma_{LY}}{\sigma_{LL}}, \quad 
    \lambda_{TA} = \frac{\sigma_{TA}}{\sigma_{TT}}, \quad 
     \lambda_{AY} = \frac{\sigma_{LT}\sigma_{LY} - \sigma_{LL}\sigma_{TY}}{\sigma_{LT}\sigma_{LA} - \sigma_{LL}\sigma_{TA}},
\]
where, for example, the causal effect $\lambda_{AY}$ corresponds to the directed edge $A \rightarrow Y$. However, the effect $\lambda_{AY}$ is also identified via the formula
\[
    \lambda_{AY} = \frac{\sigma_{TY} - \lambda_{LY}\sigma_{LT}}{\sigma_{TA}}
\]
whenever the effect $\lambda_{LY}$ is identified beforehand. 
By clearing the denominators of both identification formulas, we obtain two ``identifying polynomials'' for the effect $\lambda_{AY}$, which are as follows:
\[
    \lambda_{AY} (\sigma_{LT}\sigma_{LA} - \sigma_{LL}\sigma_{TA}) - \sigma_{LT}\sigma_{LY} + \sigma_{LL}\sigma_{TY} \quad \text{ and } \quad \lambda_{AY} \sigma_{TA} - \sigma_{TY} + \lambda_{LY}\sigma_{LT}.
\]
Note that the degree of the second polynomial is lower. Our algorithm finds identifying polynomials of the lowest possible degree, resulting in significantly more efficient computations. 
\end{example}

The organization of the paper is as follows. In Section~\ref{sec:algebraic-setup}, we introduce necessary tools and provide a definition of rational identifiability. In Section~\ref{sec:complexity}, we recall complexity results for Gröbner bases, and in Section~\ref{sec:weighting-and-homogenization} we establish how to check rational identifiability with homogeneous equations.  Based on this, we then present our main result and our identification algorithm in Section~\ref{sec:main-result}. In Section~\ref{sec:experiments}, we compare our algorithm with the standard symbolic algorithm  by \citet{garcia2010identifying} in numerical experiments. Finally, in Section~\ref{sec:comparison-to-existing-criteria}, we relate our algorithm to the polynomial-time sufficient criteria that exist in the literature. The Appendix contains the proofs of all results.

\section{Rational Identifiability} \label{sec:algebraic-setup} 

Let $G=(V,D,B)$ be a mixed graph, where $V=\{1, \ldots, p\}$ is the  set of nodes and $D, B \subseteq V \times V$ are two sets of edges. We say that an element $(v,w) \in D$ is a directed edge, and we represent it as $v \rightarrow w \in D$. We say that an element $(v,w) \in B$ is a bidirected edge, and we assume that bidirected edges have no orientation, that is, $(v,w) \in B$  if and only if $(w,v) \in B$. We represent a bidirected edge as $v \leftrightarrow w$. Neither the directed part nor the bidirected part contain self-loops, that is, $v \rightarrow v \not\in D$ and $v \leftrightarrow v \not\in B$ for all $v \in V$. In this paper, we restrict ourselves to acyclic mixed graphs, in which the directed parts do not contain cycles.

For the purpose of deciding rational identifiability, we may identify the linear structural equation model with a set of covariance matrices of the form~\eqref{eq:cov}. To formally define this, we first introduce the necessary notation, which is from \citet{foygel2012halftrek}. We write $\mathbb{R}^D$ for the set of real $p \times p$ matrices $\Lambda=(\lambda_{wv})$ with support $D$, that is $\lambda_{wv} = 0$ if $w \rightarrow v \not\in D$. If the graph $G$ is acyclic, then the matrix $I-\Lambda$ is invertible for all $\Lambda \in \mathbb{R}^D$, where $I$ denotes the $p \times p$ identity matrix. Finally, we write $\text{PD}(p)$ for the cone of positive definite $p \times p$ matrices $\Omega=(\omega_{wv})$ and we let $\text{PD}(B)$ be the subcone of matrices with support $B$, that is, $\omega_{wv} = 0$ if $w \leftrightarrow v \not\in B$. \\

\begin{definition} \label{def:linear-sem}
The linear structural equation model given by the acyclic mixed graph $G=(V,D,B)$  with $V=\{1, \ldots, p\}$ is the set of all $p \times p$ covariance matrices 
\[
    \Sigma = (I - \Lambda)^{-\top} \Omega  (I - \Lambda)^{-1}
\]
for $\Lambda \in \mathbb{R}^D$ and $\Omega \in \text{PD}(B)$.
\end{definition}

A linear structural equation model is identifiable if the parameter matrices $\Lambda \in \mathbb{R}^D$ and  $\Omega \in \text{PD}(B)$ can be uniquely recovered from $\Sigma$. In other words, identifiability holds if the parametrization
\begin{equation} \label{eq:param-map}
(\Lambda,\Omega)\mapsto (I - \Lambda)^{-\top} \Omega  (I - \Lambda)^{-1}
\end{equation}
is injective on the domain $  \mathbb{R}^D \times \text{PD}(B)$, or on a dense open subset. Since the parametrization is a rational function, the inverse, if it exists, is an algebraic function. Most interest in the literature is given to \emph{rational identifiability}, which refers to settings where the inverse is also given by a rational function. That is, each direct causal effect $\lambda_{wv}$ is identified by a rational formula, i.e., $\lambda_{wv} = b(\sigma)/a(\sigma)$, where $a$ and $b$ are polynomials in the entries of the covariance matrix $\Sigma=(\sigma_{ij})$, and where $a(\sigma)$ is not the zero polynomial.

We now explain how methods from computational algebra can be used for deciding whether a given graph allows for rational identifiability of the parameters. Let $\mathbb{R}[x_1,\ldots, x_m]$ be a polynomial ring, i.e., the set of all polynomials in the indeterminates $x_1, \ldots, x_m$ with real coefficients. In our setting, we work with indeterminates corresponding to the matrices $\Lambda, \Omega$ and $\Sigma$. To define this, let $\lambda = \{\lambda_{uv}: u \rightarrow v \in D\}$  be  indeterminates corresponding to the directed edges and let $\omega = \{\omega_{uv}: u \leftrightarrow v \in B, u < v\} \cup \{\omega_{vv}: v \in V\}$ be indeterminates corresponding to the bidirected edges. 
For convenience, we will denote by $\theta = \lambda \cup \omega$ the set all indeterminates corresponding to the directed and the bidirected edges. Moreover, let $\sigma=\{\sigma_{uv}: 1 \leq u \leq v \leq p\}$ be indeterminates representing the entries of the covariance matrix. From now on, we will consider the matrix $\Lambda$ as a matrix of symbolic indeterminates with entries $\Lambda_{uv} = \lambda_{uv}$ if $u \rightarrow v \in D$ and $\Lambda_{uv} = 0$ else. Similarly, the matrix $\Omega$ has entries $\Omega_{uv} = \Omega_{vu} = \omega_{uv}$ if $u \leftrightarrow v \in B$  or $u=v$, and $\Omega_{uv} = 0$ else. The matrix $\Sigma$ has entries $\Sigma_{uv}=\sigma_{uv}$ for $u \leq v$ and $\Sigma_{uv}=\sigma_{vu}$ else.
It is useful to  consider the mapping between polynomial rings
\begin{align*}
    \tau_G: \,\,\,\,\,\,\, \mathbb{R}[\sigma] &\longrightarrow \mathbb{R}[\theta] \\
    \sigma_{uv} &\longmapsto [(I-\Lambda)^{-\top} \Omega (I-\Lambda)^{-1}]_{uv},
\end{align*}
which is dual to the parametrization in~\eqref{eq:param-map}. Acyclicity of $G$ ensures that $\tau_G(\sigma_{uv})$ is a polynomial. This polynomial exhibits a nice combinatorial structure given by treks. A trek in a mixed graph $G$ is a walk whose consecutive edges do not have colliding arrowheads. In other words, a trek between nodes $v_l$ and $u_r$ is of one of the two following forms:
\[
    v_l\leftarrow \cdots \leftarrow v_1 \leftrightarrow u_1 \rightarrow \cdots  \rightarrow u_r ,
\]
or 
\[
    v_l\leftarrow \cdots \leftarrow v_1 = u_1 \rightarrow \cdots  \rightarrow u_r ,
\]
where $l,r \geq 1$. The nodes $v_1, \ldots, v_l$ are pairwise distinct and the nodes $u_1, \ldots, u_r$ are pairwise distinct, but the sets $\{v_1, \ldots, v_l\}$ and $\{u_1, \ldots, u_r\}$ are allowed to intersect. We refer to  \citet{sullivant2010trek} and references therein for more details on treks. For any trek $\pi$, we associate a trek monomial given by
\[
\pi(\theta) = \omega_{v_1u_1} \prod_{k=1}^{l-1} \lambda_{v_{k}v_{k+1}}\prod_{k=1}^{r-1} \lambda_{u_k u_{k+1}} \in \mathbb{R}[\theta].
\]
Denote by $\mathcal{T}(v,w)$ the set of all treks from $v$ to $w$. Then, the \emph{trek rule} \citep{spirtes2000causation, wright1934method} 
says that the image $\tau_G(\sigma_{uv})$ is given as the summation over all treks, i.e.,
\begin{equation*} 
    \tau_G(\sigma_{uv}) =  [(I-\Lambda)^{-\top} \Omega (I-\Lambda)^{-1}]_{uv} = \sum_{\pi \in \mathcal{T}(u,v)} \pi(\theta).
\end{equation*}

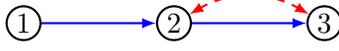
\begin{figure}[t]
\centering
\tikzset{
  every node/.style={circle, inner sep=0.3mm, minimum size=0.45cm, draw, thick, black, fill=white, text=black},
  every path/.style={thick}
}
\begin{tikzpicture}[align=center]

  \node[] (1) at (-4,0) {$1$};
  \node[] (2) at (-2,0) {$2$};
  \node[] (3) at (-0,0) {$3$};
  
  \draw[blue] [-latex] (1) edge (2);
  \draw[blue] [-latex] (2) edge (3);
   \draw[red, dashed, bend left] [latex-latex] (2) edge (3);
\end{tikzpicture}    
\caption{Mixed graph for the instrumental variable model.}
\label{fig:instrumental-variable}  
\end{figure}

\begin{example} \label{ex:trek}
    Consider the graph in Figure~\ref{fig:instrumental-variable} with nodes $V=\{1,2,3\}$. Then, the treks between $2$ and $3$ are given by $2 \leftrightarrow 3$, $2 \rightarrow 3$, and $2 \leftarrow 1 \rightarrow 2 \rightarrow 3$. Hence, we have that
    \[
    \tau_G(\sigma_{2 3}) = \omega_{2 3} + \omega_{2 2} \lambda_{2 3} + \omega_{1 1} \lambda_{1 2}^2 \lambda_{2 3}.
    \]
\end{example}

Now, we return to studying identifiability of the parameters $\theta$. An algebraic definition of rational identifiability is given as follows.

\begin{definition}
The parameter $q \in \theta$ is \emph{rationally identifiable} if there exist polynomials $a,b \in \mathbb{R}[\sigma]$ such that $a \not\in \ker(\tau_G)$ and $\tau_G(b)/ \tau_G(a) = q$. The mixed graph $G=(V,D,B)$ is said to be \emph{rationally identifiable} if all parameters $q \in \theta$ are rationally identifiable.
\end{definition}

The requirement that $a$ is not in $\ker(\tau_G)$ guarantees that $\tau_G(a)$ is not the zero polynomial, ensuring that we do not divide by zero.

\begin{example}
We return to the graph in Figure $2$. The parameter $\lambda_{23}$ is rationally identifiable by taking $b(\sigma)=\sigma_{13}$ and $a(\sigma)=\sigma_{12}$. This can be seen by calculating $\tau_G(b)=\tau_G(\sigma_{13}) = \omega_{11}\lambda_{12}\lambda_{23}$ and $\tau_G(a)=\tau_G(\sigma_{12}) = \omega_{11}\lambda_{12}$, which implies that $\tau_G(b)/\tau_G(a) = \lambda_{23}$. Moreover, we note that $\tau_G(a)$ is not the zero polynomial.
\end{example}

\begin{remark} \label{rem:lambdas-enough}
If all parameters $\lambda \subseteq \theta$ are rationally identifiable, then all remaining parameters $\omega = \theta \setminus \lambda$ are rationally identifiable since
\(
    \Omega = (I-\Lambda)^{\top} \Sigma (I - \Lambda)
\)
by Definition~\ref{def:linear-sem}.
\end{remark}

We now introduce some core concepts from algebra which we will use throughout the remainder of this paper. A subset  of a polynomial ring $\mathbb{R}[x_1,\ldots, x_m]$ is called an \emph{ideal} if it is closed under addition and
under multiplication by an arbitrary polynomial;  see \citet[Section 1.4]{cox2015ideals} for a precise definition. A basic example of an ideal is one that is generated by a given set of polynomials. For polynomials $f_1, \ldots, f_p \in \mathbb{R}[x_1,\ldots, x_m]$, we write
\[
    \langle f_1, \ldots, f_p \rangle := \left\{ \sum_{i=1}^p g_i f_i : g_i \in \mathbb{R}[x_1,\ldots, x_m] \right\},
\]
that is, the collection of all polynomial combinations of the $f_i$. By Hilbert’s basis theorem, every ideal can be expressed as one generated by finitely many polynomials. Working with ideals allows us to check rational identifiability of a given parameter $\lambda_{uv}$ by checking whether a certain ideal  $\mathcal{I}$ contains an \emph{identifying polynomial}. \\

We will work with the ideal $\mathcal{I} \subseteq \mathbb{R}[\theta, \sigma]$ of the graph of the parametrization $\tau_G$. It is generated by the polynomials
\begin{equation*} 
    \sigma_{uv} - \tau_G(\sigma_{uv}) \text{ for } u \leq v.
\end{equation*}

Note that a polynomial $g(\theta, \sigma ) \in \mathbb{R}[\theta, \sigma]$ is an element of $\mathcal{I}$ if and only if $g(\theta, \tau_G(\sigma))= 0$, where $\tau_G(\sigma)$ is the element-wise application of the function $\tau_G$ to the variables in $\sigma$.  The following result is Lemma 7 in \citet{foygel2012halftrekSUPP}; also see \citet{garcia2010identifying} for a more detailed discussion.  

\begin{lemma}[\citeauthor{foygel2012halftrekSUPP}, \citeyear{foygel2012halftrekSUPP}, Lemma 7] \label{lem:chracterization-identifiability-1}
    The parameter $q \in \theta$ is rationally identifiable if and only if $\mathcal{I}$ contains an element of the form $q a(\sigma)  - b(\sigma)$ with $a,b \in \mathbb{R}[\sigma]$ and $a \not\in \ker(\tau_G)$.
\end{lemma}

In Lemma~\ref{lem:chracterization-identifiability-1}, note that $\ker(\tau_G)$ is given by $\mathcal{I} \cap \mathbb{R}[\sigma]$. Since we want to keep the degree of the polynomials $a$ and $b$ as low as possible, we will allow them  to depend on the already identified parameters. For this, we show a refined version of Lemma~\ref{lem:chracterization-identifiability-1}. Let $\theta_{\id} \subseteq \theta$ be a subset of parameters that is already known to be rationally identifiable, and let $\theta_{\rem} := \theta \setminus \theta_{\id}$ be the remaining parameters.

\begin{lemma} \label{lem:chracterization-identifiability-2}
    The parameter $q \in \theta_{\rem}$ is rationally identifiable if and only if  the ideal $\mathcal{I}$ contains an element of the form $q a(\theta_{\id}, \sigma) - b(\theta_{\id}, \sigma)$ with $a,b \in \mathbb{R}[\theta_{\id}, \sigma]$ and $a \not \in \mathcal{I} \cap \mathbb{R}[\theta_{\id}, \sigma]$.
\end{lemma}

Lemma~\ref{lem:chracterization-identifiability-2} justifies the following definition.

\begin{definition} \label{def:identifying-polynomial}
    Let $\prec$ be a total order on $\theta$. We say that $g_q \in \mathcal{I}$ is an \emph{identifying polynomial} for $q \in \theta$ with respect to ~the order $\prec$ if there is a  subset $\theta_{\id} \subseteq \theta$ with $s \prec q$ for all $s \in \theta_{\id}$ such that $g_q$ is of the form $g_q = q a(\theta_{\id}, \sigma) - b(\theta_{\id}, \sigma)$ with $a,b \in \mathbb{R}[\theta_{\id}, \sigma]$ and $a \not \in \mathcal{I} \cap \mathbb{R}[\theta_{\id}, \sigma]$. 

    If an identifying polynomial  $g_q$ exists for all $q \in \theta$ with respect to the order $\prec$, then we say that $\prec$ is an \emph{identifying order}, and we denote 
    $P_{\prec}=\{g_q: q \in \theta\}$.
\end{definition}

Note that, for a given identifying order, there might exist multiple sets of identifying polynomials $P_{\prec}$. For any such set of of identifying polynomials it holds that $|P_{\prec}|=|\theta|$.

\begin{example} \label{ex:identifying-order}
For the graph in Figure~\ref{fig:another-example}, we verified with Algorithm~\ref{alg:degree-bounded-identification} that the ideal $\mathcal{I}$ contains the three polynomials 
\begin{equation} \label{eq:ex-id-polynomials}
    \lambda_{12} \sigma_{11} - \sigma_{12}, \qquad
    \lambda_{14} \sigma_{11} - \sigma_{14}, \qquad
    \lambda_{34} \sigma_{23} + \lambda_{12} \sigma_{14} - \sigma_{24}.
\end{equation}
Hence, any order on the parameters for which $\lambda_{12}$ is smaller than $\lambda_{34}$ and all parameters in $\omega$ are larger than the parameters in $\lambda$ is an identifying order. For such an order, identifying polynomials for the parameters in $\lambda$ are given in~\eqref{eq:ex-id-polynomials}, and identifying polynomials for the parameters in $\omega$ are then given via the formula in Remark~\ref{rem:lambdas-enough}.
\end{example}

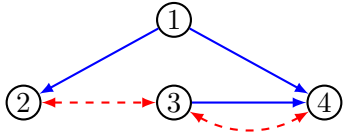
\begin{figure}[t]
\centering
\tikzset{
  every node/.style={circle, inner sep=0.3mm, minimum size=0.45cm, draw, thick, black, fill=white, text=black},
  every path/.style={thick}
}
\begin{tikzpicture}[align=center]
 \node[] (0) at (-2,1.1) {$1$};
  \node[] (1) at (-4,0) {$2$};
  \node[] (2) at (-2,0) {$3$};
  \node[] (3) at (-0,0) {$4$};
  
  \draw[blue] [-latex] (0) edge (1);
  \draw[blue] [-latex] (0) edge (3);
  \draw[blue] [-latex] (2) edge (3);
   \draw[red, dashed, bend right] [latex-latex] (2) edge (3);
   \draw[red, dashed] [latex-latex] (1) edge (2);
\end{tikzpicture}    
\caption{Another mixed graph.}
\label{fig:another-example}  
\end{figure}

Clearly, a mixed graph $G$ is rationally identifiable if and only if there exists an identifying order. In general, there might be multiple identifying orders on $\theta$. We will later derive an algorithm that finds the one such that the total degree of any identifying polynomial that appears is minimal. 

\begin{definition}
    Let $G$ be rationally identifiable and suppose that $P_{\prec}$ is a set of identifying polynomials with respect to the identifying order $\prec$. We denote by $\deg(P_{\prec})$ the maximal total degree of the polynomials in $P_{\prec}$. Moreover, we say that $\text{ID}_G := \min_{\prec} \min_{P_{\prec}} \deg(P_{\prec})$ is the  \emph{identifying degree}, where the first minimum is taken over all identifying orders and the second minimum is taken over all sets of identifying polynomials for a given order.
\end{definition}

\section{Complexity of Efficient Gr\"obner Basis Algorithms} \label{sec:complexity}
Gr\"obner bases are at the core of most algorithms in computational algebra. We now discuss how they can be used to determine if identifying polynomials exist in $\mathcal{I}$ and how to compute them.

\subsection{Gröbner Bases}
An accessible introduction to Gröbner bases is provided in \citet{cox2015ideals}. We briefly outline the key properties,  orienting ourselves to the exposition in \citet{garcia2010identifying}. \\

For an integer vector $u \in \mathbb{N}^m$, we denote a monomial in $\mathbb{R}[x_1, \ldots, x_m]$ by $\mathbf{x}^u=x_1^{u_1} x_2^{u_2} \cdots x_m^{u_m}$. A monomial order $\prec$ on the polynomial ring $\mathbb{R}[x_1, \ldots, x_m]$ is a total ordering of all monomials in the ring, which respects multiplication and sets $1$  as the smallest monomial. Concretely, this means that whenever $\mathbf{x}^u \preceq \mathbf{x}^v$ then  $\mathbf{x}^w \cdot \mathbf{x}^u \preceq \mathbf{x}^w \cdot \mathbf{x}^v$ and that $1 \preceq \mathbf{x}^u$ for all $u \in \mathbb{N}^n$.  Because $\prec$ is a total order, each polynomial $f \in \mathbb{R}[x_1, \ldots, x_m]$ contains a uniquely determined largest monomial, which we denote by $\text{in}_{\prec}(f)$ and which we call the initial term or the leading term. For an ideal  $I \subseteq \mathbb{R}[x_1, \ldots, x_m]$, we define its initial ideal by $\text{in}_{\prec}(I)=\langle \text{in}_{\prec}(f): f \in I \rangle $. 

\begin{definition} \label{def:gb}
A finite set $G \subseteq I$ is called a \emph{Gröbner basis} of $I$ (with respect to $\prec$) if
\[
    \text{in}_{\prec}(I) = \langle \text{in}_{\prec}(f): f \in G \rangle.
\]
A Gröbner basis is said to be \emph{reduced} if, for all $f \in G$, the coefficient of the leading term $\text{in}_{\prec}(g)$ is $1$ and no monomial of $f$ is divisible by any other leading term $\text{in}_{\prec}(g)$ with $g \in G \setminus \{f\}$.
\end{definition}

Gröbner bases have many nice properties including that the reduced Gröbner basis of an ideal is unique for a fixed monomial order. For the purpose of determining identifiability, we will use lexicographic and elimination orders \citep[Chapter 3]{ene2012groebner}, which we explain now.

In the lexicographic order, to  decide whether $\mathbf{x}^u \prec \mathbf{x}^v$, we examine the vector $v-u$ and find the first nonzero entry; then we declare $\mathbf{x}^u \prec \mathbf{x}^v$ exactly when that entry is positive.  Intuitively, the variable $x_1$ has the highest priority, so its degree determines the comparison. If two monomials have the same degree in $x_1$, we then compare their degrees in $x_2$, and proceed similarly through the remaining variables.

Elimination orders generalize the lexicographic order. Suppose that the variables $x_1, \ldots, x_m$ are partitioned into two blocks $B_1 \cup B_2$. Then, a monomial order on  $\mathbb{R}[x_1, \ldots, x_m]$ is an \emph{elimination order} for $B_1$ if $\mathbf{x}^{u} \prec \mathbf{x}^{v}$ whenever $ \mathbf{x}^{v}$ has larger degree in the $B_1$ variables than $\mathbf{x}^{u}$. If the two monomials  $\mathbf{x}^{u}$ and $\mathbf{x}^{v}$ have the same degree in the $B_1$ variables, then another term order is used to break ties.  The following is a more formal definition.

\begin{definition} \label{def:lex-elim-orders}
Let $\mathbb{R}[x_1, \ldots, x_m]$ be a polynomial ring with variables partitioned into blocks $B_1$ and $B_2$. For each block $i=1,2$, let $\prec_i$ be a monomial order on $\mathbb{R}[B_i]$. For a monomial $\mathbf{x}^u \in \mathbb{R}[x_1, \ldots, x_m]$, denote by $\mathbf{x}^{u_{B_i}}$ the restriction of $\mathbf{x}^u$ to the variables in block $B_i$.
A monomial order $\prec$ on $\mathbb{R}[x_1, \ldots, x_m]$ is called an \emph{elimination order for $B_1$} if  $ \mathbf{x}^v \succ \mathbf{x}^u$ whenever $\mathbf{x}^{v_{B_1}} \succ_1 \mathbf{x}^{u_{B_1}}$ or $\mathbf{x}^{v_{B_1}} = \mathbf{x}^{u_{B_1}}$ and $\mathbf{x}^{v_{B_2}} \succ_2 \mathbf{x}^{u_{B_2}}$. 
Moreover, we say that an elimination order $\prec$ for $B_1$ is a \emph{lex-elimination order} if $\prec_1$ is the lexicographic order for some permutation of the variables in $B_1$.
\end{definition}

Lex-elimination orders are useful for determining identifiability via Gröbner bases. Recall that we denote  by $\theta_{\id} \subseteq \theta$ the already identified parameters, and by $\theta_{\rem} = \theta \setminus \theta_{\id}$ the remaining parameters. To check the condition from Lemma~\ref{lem:chracterization-identifiability-2}, that is, whether the ideal $\mathcal{I}$ contains an identifying polynomial for the parameters $q \in \theta_{\rem}$, we have the following result.

\begin{proposition} \label{prop:id-with-gb}
Let $F$ be a reduced Gröbner basis of the ideal $\mathcal{I}$ with respect to a lex-elimination order on $\mathbb{R}[\theta, \sigma]$ for the variables in $\theta_{\rem}$. Then, all parameters $q \in \theta_{\rem}$ are rationally identifiable if and only if, for each parameter $q\in \theta_{\rem}$, the Gröbner basis $F$ contains an element that has leading term $q \cdot a(\theta_{\id}, \sigma)$  for  some $a \in \mathbb{R}[\theta_{\id}, \sigma]$.
\end{proposition}

\begin{example}
For the graph in Figure~\ref{fig:another-example}, suppose that the parameters $\theta_{\id}=\{\lambda_{12}, \lambda_{14}\}$ are already identified. For the remaining parameters $\theta_{\rem}$, consider the lex-order for $(\omega_{11}, \omega_{22}, \omega_{33}, \omega_{44}, \omega_{23}, \allowbreak \omega_{34}, \lambda_{34})$, that is, the variable $\lambda_{34}$ is the smallest and has lowest priority. Since we have already seen in Example~\ref{ex:identifying-order} that the graph is rationally identifiable, the Gröbner basis with respect to any such lex-elimination order for the variables in $\theta_{\rem}$ contains the element 
\[
     \lambda_{34} \sigma_{23} + \lambda_{12} \sigma_{14} - \sigma_{24}.
\]
\end{example}

The standard way to compute Gröbner bases is via Buchberger's algorithm, which is an extension of the division algorithm for univariate polynomials. It proceeds by canceling leading terms until no term in the remainder can be divided by the leading term of the divisor. Unfortunately, the computational complexity can be double exponential in the number of input variables  \citep{mayr1997somcecomplexity}. Hence, using Proposition~\ref{prop:id-with-gb} becomes infeasible in practice already for graphs on $5$ to $6$ nodes. However, computing Gröbner bases of homogeneous ideals is much more efficient since we can compute the Gröbner basis degree by degree.

\subsection{Homogeneous Ideals}

In this section, we recall several complexity results for computing Gr\"obner bases of homogeneous ideals, which we will use in the remainder of this paper. A polynomial is homogeneous if all its terms have the same total degrees. An ideal is homogeneous if it can be generated by homogeneous polynomials \citet[Section 8.3]{cox2015ideals}. When working with homogeneous ideals, we consider graded monomial orders, which are monomial orders that first compare monomials by total degree and only break ties using a secondary order. \\

Now, suppose that we compute the reduced Gröbner basis $\mathcal{G}_{\prec}$ of a homogeneous ideal via Buchberger's algorithm with respect to a graded monomial ordering $\prec$. Due to the grading of the ordering, we always process lower degree polynomials first, and, importantly, when all elements of degree $\leq d$ have been processed, all elements of the Gröbner basis in degrees $\leq d$ have been found \citep[Section 10.1]{cox2015ideals}. In contrast, when computing Gröbner bases of non-homogeneous ideals, lower-degree elements can appear at any stage. In this way, we can find  the set of all elements of $\mathcal{G}_{\prec}$ of degree at most $d$, which we denote by $\mathcal{G}_{\prec,d}$.  \\

This idea is leveraged by Faugére's F4 algorithm \citep{faugere1999F4, faugere2002F5} which uses linear algebra to greatly speed up Gr\"obner basis computation. More recently, it was observed that computing Gröbner bases becomes simpler in weighted homogeneous systems \citep{faugere2016complexity}. A weighting in a polynomial ring $\mathbb{R}[x_1, \ldots, x_m]$ is given by a weight vector $w=(w_1, \ldots, w_m) \in \mathbb{N}^m$. The weighting $w$ defines a $w$-grading of the ring by setting the degree of the variables to its weight. That is, the weighted degree of a monomial is given by $\deg_w(x_1^{u_1} x_2^{u_2} \cdots x_m^{u_m}) = \sum_{i=1}^m w_i u_i$. Moreover, a polynomial $f$ is said to be homogeneous with respect to the weighting $w$, or $w$-homogeneous, if the weighted degree of all the monomials in $f$ coincide. In this setting, we have the following complexity bound for computing Gröbner bases up to the weighted degree $d_w$.

\begin{proposition}[\citeauthor{faugere2016complexity}, \citeyear{faugere2016complexity}, Section 5.1] \label{prop:deg-bounded-gb-complexity}
Suppose that $w=(w_1, \ldots, w_m) \in \mathbb{N}^m$ is a weighting. 
Let $I = \langle f_1, \ldots, f_p \rangle \subseteq \mathbb{R}[x_1, \ldots, x_m]$ be a $w$-homogeneous ideal and let $\prec$ be a $w$-graded monomial ordering.  The complexity (in terms of arithmetic operations in $\mathbb{R}$) to compute $\mathcal{G}_{\prec, d_w}$  is bounded by
\begin{equation*} 
O \left(\frac{1}{(\prod_{i=1}^m w_i)^\alpha} \binom{m+d_w-1}{d_w}^\alpha \right),
\end{equation*}
where $\alpha$ is a constant such that row reduction of a $n \times n$ matrix can be performed in $O(n^{\alpha})$ operations. 
\end{proposition}

This complexity result follows naturally from the standard encoding of Gr\"obner basis computation into linear algebra which state-of-the-art algorithms such as F4 and F5 leverage. These algorithms are typically implemented in such a way that they compute degree-by-degree as described above. Observe that since $O(\binom{m+d_w-1}{d_w}) = O((m+d_w)^{d_w})$, the computation of $G_{\prec, d_w}$ is actually polynomial in the number of variables $m$, but still grows exponentially in the degree $d_w$. 
In the context of identification, this means that if we know a-priori that there is a set of homogenized identifying polynomials for a graph $G$ such that all of them are of degree at most $d_w$ for a specific weighting, then these identifying polynomials can be computed in polynomial time in the size of the vertex and edge set of $G$. This will be the main focus of our next section.

\section{Weighting and Homogenization} \label{sec:weighting-and-homogenization}
As shown in Section \ref{sec:algebraic-setup}, deciding identifiability of a parameter $q \in  \theta$ is equivalent to deciding whether $\mathcal{I}$ contains an identifying polynomial. In this section, we transfer this problem to an ideal that is homogeneous with respect to a specific weighting. We set the weight of $\sigma_{uv}$  to $w(\sigma_{uv}) = \deg(\tau_G(\sigma_{uv}))$. Observe that $\deg(\tau_G(\sigma_{uv}))$ is equal to the maximal length of any trek between $u$ and $v$, where the length of a trek $\pi$ is the total degree of the trek monomial $\pi(\theta)$. If there is no trek between two nodes $u$ and $v$, then we set the maximal trek length between $u$ and $v$ to $1$. In this way, we define the \emph{trek weighting} $w$ in the ring $\mathbb{R}[\theta, \sigma]$ by setting 
\begin{equation} \label{eq:trek-weighting}
\begin{aligned}
w(q) &= 1 \text{ for all } q \in \theta, \\
w(\sigma_{uv}) &= \text{ maximal length of any trek between }u \text{ and }v.
\end{aligned}
\end{equation}
Moreover, we denote by $w_\mathrm{trek}$ to be the length of the longest trek in $G$.  Now, we consider the $w$-homogenized ideal $\mathcal{I}^{\wh}=\langle f^{\wh}: f \in \mathcal{I} \rangle \subseteq \mathbb{R}[\theta, \sigma, h]$, where $f^{\wh}$ denotes the weighted homogenization with respect to a new variable $h$ with weight $w(h)=1$. The weighted homogenization of a polynomial $f(x_1, \ldots, x_m) \in \mathbb{R}[x_1, \ldots, x_m]$ is given by $f^{\wh}(x_1, \ldots, x_m,h) = h^{d_w} f(x_1/h^{w_1}, \ldots, x_m/h^{w_m})$, where $d_m$ is the $w$-degree of $f$ and $w=(w_1, \ldots,w_m)$ is a weighting. Checking identifiability will then correspond to deciding whether $\mathcal{I}^{\wh}$ contains a homogenized identifying polynomial. The advantage is that we can organize our Gröbner basis computations degree-by-degree and apply the efficiency result given in Proposition~\ref{prop:deg-bounded-gb-complexity}. Recall that for a $w$-homogeneous ideal $J$, the dehomogenized ideal is given by $J^{\deh}=\{f^{\deh}:f \in J\}$, where $f^{\deh}$ is the usual dehomogenization obtained by plugging-in $1$ for the variable $h$. In our proofs, we use fundamental properties of homogenization and dehomogenization as described in \citet[Section 4.3]{kreuzer2005computational}. In particular, $(\mathcal{I}^{\wh})^{\deh}=\mathcal{I}$ and $f \in \mathcal{I}^{\wh}$ if and only if $f^{\deh} \in \mathcal{I}$ for a weighted homogeneous polynomial $f$. \\

The next proposition shows that $\mathcal{I}^{\wh}$ is generated by the polynomials $\sigma_{uv}-\tau_G(\sigma_{uv})^{\wh}$, which are $w$-homogeneous by the definition of the trek-weighting.

\begin{proposition} \label{prop:hom-equal}
Let $\bar{\mathcal{I}} \subseteq \mathbb{R}[\theta, \sigma, h]$ be the ideal generated by $\sigma_{uv}-\tau_G(\sigma_{uv})^{\wh}$, $u \leq v$. Then, $\bar{\mathcal{I}}$ is homogeneous with respect to the trek weighting, and it holds that
\begin{itemize}
    \item[(i)] $\bar{\mathcal{I}}=\mathcal{I}^{\wh}$, and \item[(ii)] $\mathcal{I}^{\wh} \cap \mathbb{R}[\theta_{\id}, \sigma ,h]= (\mathcal{I} \cap \mathbb{R}[\theta_{\id}, \sigma])^{\wh}$ for any subset $\theta_{\id} \subseteq \theta$.
\end{itemize}
\end{proposition}

Now, checking identifiability can be done in $\mathcal{I}^{\wh}$. This is shown in the following lemma. 

\begin{lemma} \label{lem:chracterization-identifiability-hom}
    The parameter $q \in \theta_{\rem}$ is rationally identifiable if and only if the ideal $\mathcal{I}^{\wh}$ contains an element of the form $q a(\theta_{\id}, \sigma,h) - b(\theta_{\id}, \sigma,h)$ with $a,b \in \mathbb{R}[\theta_{\id}, \sigma,h]$ and $a \not \in \mathcal{I}^{\wh} \cap \mathbb{R}[\theta_{\id}, \sigma,h]$.
\end{lemma}

We say that the polynomial $q a(\theta_{\id}, \sigma,h) - b(\theta_{\id}, \sigma, h)$ is a \emph{homogenized identifying polynomial} for $q \in \theta_{\rem}$. 
By computing a reduced Gröbner basis, we can check whether $\mathcal{I}^{\wh}$ contains such a polynomial. For this, we define $w$-graded elimination orders, which we will use to compute Gröbner bases of $\mathcal{I}^{\wh}$. These orders first compare monomials by their weighted degree and then break ties using an elimination order.

\begin{definition} \label{def:block-orders}
Let $\mathbb{R}[x_1, \ldots, x_m]$ be a polynomial ring with variables partitioned into blocks $B_1$ and $B_2$. Suppose that $w=(w_1, \ldots, w_m) \in \mathbb{N}^m$ is a weighting. A monomial order on $\mathbb{R}[x_1, \ldots, x_m]$ is a $w$-graded (lex-)elimination order, if we first compare monomials by their weighted degree, and break ties using a (lex-)elimination order.
\end{definition}

\begin{theorem}\label{thm:id-with-gb}
    Let $F$ be a reduced Gröbner basis of the ideal $\mathcal{I}^{\wh} \subseteq \mathbb{R}[\theta, \sigma, h]$ with respect to a $w$-graded lex-elimination order for $\theta_{rem}$, where the weighting $w$ is the trek-weighting.
    Then, all parameters $q \in \theta_{\rem}$ are rationally identifiable if and only if, for each $q \in \theta_{\rem}$, the basis $F$ contains an element that has leading term $q \cdot a(\theta_{\id}, \sigma, h)$  for  some $a \in \mathbb{R}[\theta_{\id}, \sigma, h]$.

\end{theorem}

Of course, by Lemma~\ref{lem:chracterization-identifiability-hom}, whenever the Gröbner basis from Theorem~\ref{thm:id-with-gb} contains an identifying polynomial $q  a(\theta_{\id}, \sigma, h) - b(\theta_{\id}, \sigma, h)$, then the parameter $q \in \theta$ is rationally identifiable. Now, a  key observation is that if there is a parameter $q \in \theta_{\rem}$ that is rationally identifiably, then it is usually not necessary to compute the full Gröbner basis $F$ in Theorem~\ref{thm:id-with-gb}. Instead we compute the elements in reduced Gröbner basis $F$ degree by degree. As soon as we obtain an homogenized identifying polynomial for some $q \in \theta_{\rem}$, we can stop the computation. The following corollary certifies that we will find polynomials that have the lowest possible degree with such a strategy.

\begin{corollary} \label{cor:degree}
Suppose that a graph $G=(V,D, B)$ is rationally identifiable. Let $\prec$ be a total order on $\theta$ and let $P_{\prec} \subseteq \mathbb{R}[\theta,\sigma]$ be a set of identifying polynomials such that $\text{ID}_G = \deg(P_{\prec})$. Let $q \in \theta$ and suppose that the set of preceding variables is known to be rationally identifiable, that is, $\{s \in \theta: s \prec q\} = \theta_{\id}$. Moreover, suppose that the $w$-graded lex-elimination order for $\theta_{\rem}$ is such that $q$ is the smallest variable among the $\theta_{\rem}$-variables. Then, the Groebner basis $F$ from Theorem~\ref{thm:id-with-gb} contains a polynomial $g_{q} := q a(\theta_{\id}, \sigma,h) - b(\theta_{\id}, \sigma, h)$ that satisfies $\deg_w(g_{q}) \leq \text{ID}_G \cdot w_{\trek}$.
\end{corollary}

Before describing our full algorithm to verify rational identifiability of a graph $G$, we study in the next section the complexity of degree-bounded Gröbner basis computations.

\section{Degree-Bounded Identification via Adaptive Orderings} \label{sec:main-result}

A mixed graph is rationally identifiable if and only if all parameters are rationally identifiably. We now propose an algorithm that recursively certifies parameters to be rationally identifiable. At the beginning we let $\theta_{\id} = \emptyset$  and we compute all elements up to a small total degree in the reduced Gröbner basis $F$ of the ideal $\mathcal{I}^{\wh}$ with respect to a $w$-graded lex-elimination order. If $F$ contains a homogenized identifying polynomial for some $q \in \theta_{\rem}$, we add $q$ to the set $\theta_{\id}$ and remove it from the set $\theta_{\rem}$. We then refine our $w$-graded lex-elimination order accordingly, and repeat the procedure until we find the next variable that is certified to be rationally identifiable. This procedure is formalized in Algorithm~\ref{alg:degree-bounded-identification}.

\begin{algorithm}[t]
\caption{Degree-bounded identification in a mixed graph}
\begin{algorithmic}[1]
\REQUIRE Generators of the ideal $\mathcal{I}^{\wh} \subseteq \mathbb{R}[\theta, \sigma, h]$, maximal degree $d \geq 2$.
\ENSURE 
Let $\theta_{\id} = \emptyset$ be the rationally identifiable parameters and set $\theta_{\rem} = \theta \setminus \theta_{\id}$. Let $w$ be the trek-weighting. 
\STATE Let $d'=d \cdot w_{\trek}$ where $w_{\trek}$ is the length of the maximal length of a trek in $G$. 
\REPEAT
\FOR {$k = 1, \ldots, d'$}
\FOR {$q \in \theta_{\rem}$}
    \STATE Let $\prec$ be a $w$-graded lex-elimination order for $\theta_{\rem}$ s.t.~$q$ is the smallest element of $\theta_{\rem}$.
    \STATE Using a degree-bounded Gröbner basis algorithm, compute the reduced Gröbner basis $F_{\leq k}$ up to the weighted degree $k$ with respect to the order $\prec$.
    \FOR {$q \in \theta_{\rem}$}
        \IF {$q a(\theta_{\id}, \sigma,h) - b(\theta_{\id}, \sigma,h) \in F_{\leq k}$}
            \STATE $\theta_{\id} = \theta_{\id} \cup \{q\}$ and $\theta_{\rem} = \theta_{\rem} \setminus \{q\}$.
            \BREAK { all for-loops}
        \ENDIF
    \ENDFOR
\ENDFOR
\ENDFOR
\UNTIL{$\theta_{\id} = \theta$ or no change has occurred in the last iteration.}
\RETURN ``yes'' if $\theta_{\id} = \theta$, ``no'' otherwise.
\end{algorithmic}
\label{alg:degree-bounded-identification}
\end{algorithm}

\begin{theorem} \label{thm:main-result}
Let $G=(V,D,B)$ be an acyclic mixed graph. If $G$ is rationally identifiable  and $d \geq \text{ID}_G$,  then  Algorithm~\ref{alg:degree-bounded-identification}  returns ``yes''. If $G$ is not rationally identifiable, then Algorithm~\ref{alg:degree-bounded-identification} returns ``no'' for all $d \in \mathbb{N}$. The computational complexity of Algorithm~\ref{alg:degree-bounded-identification} is of order 
\[
O\left( d |V|^5 \{3d|V|\}^{4\alpha d^2\log(3d|V|)}\right),
\]
where $\alpha$ is a constant such that row reduction of a $n \times n$ matrix can be performed in $O(n^{\alpha})$ operations. Moreover, if $w_{\trek} \geq 2d \log(3d |V|)$, then the complexity of Algorithm~\ref{alg:degree-bounded-identification} is $ O(d|V|^5)$, and if $w_{\trek} \leq C$ for an absolute constant $C$, then the complexity of Algorithm~\ref{alg:degree-bounded-identification} is $O( d |V|^5 \{3d|V|\}^{2 \alpha d C})$.
\end{theorem}

Theorem~\ref{thm:main-result} certifies that, for fixed input degree $d$, Algorithm~\ref{alg:degree-bounded-identification} is quasi-polynomial time in the number of nodes of the graph. If the maximal trek length is either bounded by a constant or grows fast enough with the number of nodes, then the algorithm is polynomial time. The key ingredient to the proof of Theorem~\ref{thm:main-result} is the next lemma. It relates the product of all trek weights to the superfactorial. 

\begin{lemma} \label{lem:superfactorial}
    Let $G=(V,D,B)$ be an acyclic mixed graph. Consider the trek-weighting $w$ defined in~\eqref{eq:trek-weighting}  and let $w_{\trek}$ be the maximum length of any trek in $G$. Then, we have
    \[
        \left(\prod_{u \leq v} w(\sigma_{uv})\right)^4 \geq \text{sf}(w_{\trek}) ,
    \]
    where $\text{sf}(n) = \prod_{i=1}^n i!$ is the $n$-th superfactorial for $n \in \mathbb{N}$.
\end{lemma}

The proof of Theorem~\ref{thm:id-with-gb} is based on Lemma~\ref{lem:superfactorial} and the asymptotic behavior of the superfactorial, which is well understood. In particular, we relate the asymptotic behavior of $\text{sf}(w_{\trek})$ to $|V|^{w_{\trek}}$, depending on how fast $w_{\trek}$ growths with the number of nodes $|V|$.

\begin{remark} \label{rem:bounding-IDG}
    If it is possible to bound $\text{ID}_G$ by a constant, then our procedure provides a quasi-polynomial time algorithm to check rational identifiability. 
\end{remark}

\begin{remark}
    We can further reduce the complexity of checking rational identifiability via a graph composition due to Tian \citep{tian2005identifying}. For an acyclic mixed graph $G=(V,D,B)$, denote by $C_1, \ldots, C_k$ the pairwise disjoint node sets of the connected components of the bidirected part $(V,B)$. Let $B_i=B \cap (C_i \times C_i)$ be the corresponding bidirected edges in the $i$-th component for all $i \in [k]$. Moreover, for all $i \in [k]$, let $V_i=C_i \cup \{\pa(v): v \in C_i\}$ be the  nodes obtained as a union of $C_i$ and all parents of nodes in $C_i$. Finally, let $D_j = D \cap (V_j \times C_j)$ be the set of directed edges that point from a node $v \in V_j$ to a node $c \in C_j$. Then, the \emph{Tian decomposition} of $G$ is given by the \emph{mixed components} $G_j=(V_j, D_j,B_j)$ for $j=1, \ldots, k$; see \citet[Section 8]{foygel2012halftrek} for examples. Note that the node sets $V_1, \ldots, V_k$ need not be pairwise disjoint, but the mixed components give a partition of the edges of $G$. Crucially, an acyclic mixed graph $G$ is rationally identifiable if and only if all of its mixed components $G_1, \ldots, G_k$ are rationally identifiable \citep{tian2005identifying, foygel2012halftrek}. Therefore, when checking identifiability via Algorithm~\ref{alg:degree-bounded-identification}, we always first compute the Tian decomposition, which has low computational complexity. We then apply Algorithm~\ref{alg:degree-bounded-identification} to each component separately. Note that  if there are many mixed components in large graphs then this provides a significant reduction of the worst-case complexity stated in Theorem~\ref{thm:main-result}. 
\end{remark}

\sloppy 
We implemented Algorithm~\ref{alg:degree-bounded-identification} in the open-source computer algebra software \texttt{Macaulay2} \citep{M2}. Our code is available on GitHub: \url{https://github.com/NilsSturma/DegBoundedIdentifiability} and builds on top of the \texttt{GraphicalModels} package \citep{GraphicalModelsSource}. In our implementation, we also return the dehomogenized identifying polynomials. Moreover, we stop the algorithm as soon as identifying polynomials for all variables in $\lambda$ have been found since the identifying polynomials for the remaining variables in $\omega$ are then given via the formula in Remark~\ref{rem:lambdas-enough}. 

\begin{example}
The following \texttt{M2} code uses our implementation of Algorithm \ref{alg:degree-bounded-identification} to verify rational identifiability of the graph in Figure~\ref{fig:another-example}. 
\begin{verbatim}
load ("DegBoundedIdentification.m2")
M = {
    {
        {0, 1, 0, 1}, 
        {0, 0, 0, 0}, 
        {0, 0, 0, 1}, 
        {0, 0, 0, 0}
    }, 
    {
        {0, 0, 0, 0}, 
        {0, 0, 1, 0}, 
        {0, 1, 0, 1}, 
        {0, 0, 1, 0}
    }
}
n=4
(A1, A2) = toSequence(M / matrix);
D = digraph(toList(1..n), A1);
B = bigraph(toList(1..n), A2);
G = mixedGraph(D, B);
degBd = 2;
maxTime = 10;
tian = true;
DegBoundedIdentification(G, degBd, maxTime, tian)
\end{verbatim}
\end{example}

Lastly, we note that one might improve the implementation of our algorithm by parallelizing the for-loop in line 4 of Algorithm~\ref{alg:degree-bounded-identification}. We leave this open for future work.

\section{Numerical Experiments} \label{sec:experiments}

\begin{figure}[t]
\centering
\includegraphics[width=0.5\linewidth]{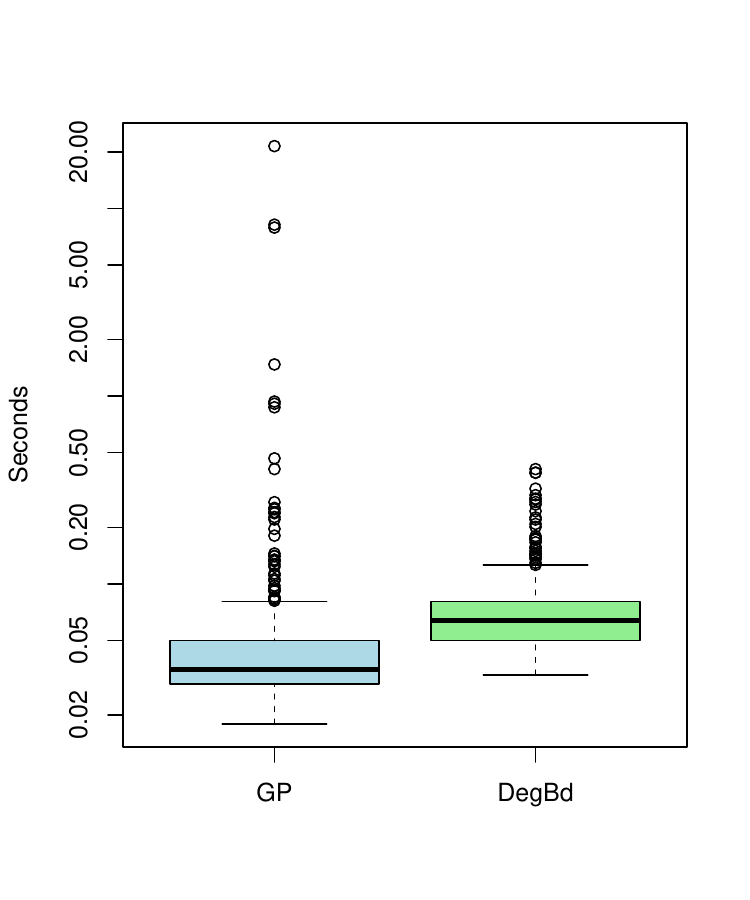}
\caption{Boxplots of computing times of the Garcia-Puente algorithm (GP) and Algorithm~\ref{alg:degree-bounded-identification} (DegBd) for certifying rationally identifiable of acyclic mixed graphs on $4$ nodes. Note that the scale is logarithmic.}
\label{fig:boxplot}
\end{figure}

We conduct two small simulation studies to demonstrate the practical applicability of Algorithm~\ref{alg:degree-bounded-identification}. We compare our algorithm to the standard method for deciding identifiability using computational algebra developed by~\citet{garcia2010identifying}. All computations were performed on a MacBook Pro (Apple M4 chip, 10-core CPU, 24 GB RAM) using the computer algebra software \texttt{Macaulay 2}, version 1.25.06; see \url{https://github.com/NilsSturma/DegBoundedIdentifiability} for the corresponding code. \\

In the first study, we consider all acyclic mixed graphs on $|V|=4$ nodes with at most $\binom{|V|}{2}=6$ edges. Note that any mixed graph with more than $\binom{|V|}{2}$ edges is trivially not generically identifiable since the number of parameters exceeds the dimension of the set of covariance matrices; compare \citet[Proposition 2]{foygel2012halftrek}. In total, there are $715$ such mixed graphs on $4$ nodes. For such graphs, the method of \citet{garcia2010identifying} that checks identifiability by computing the whole Gröbner basis in a standard way is computationally feasible. Recall that this method verifies rational identifiability as in Proposition~\ref{prop:id-with-gb}. We find that in total $343$ graphs are rationally identifiable. Using  Algorithm~\ref{alg:degree-bounded-identification} with maximal degree $d=5$, we certify the same $343$ graphs to be rationally identifiable. For those 343 graphs, Figure~\ref{fig:boxplot} shows two histograms of computation times that were needed to certify rational identifiability. On $4$ nodes, most times the Garcia-Puente algorithm computes the Gröbner basis immediately. However, there are already a handful of graphs were this method already takes considerably longer time, up to $20$ seconds. In contrast, our method via degree-bounded Gröbner bases is very fast for all graphs on $4$ nodes. The slight difference in the median runtime is mainly due to the preprocessing steps we do for Algorithm~\ref{alg:degree-bounded-identification}, such as computing the trek length grading. For larger examples, this preprocessing is negligible. Note that we were not able to carry-out the Garcia-Puente algorithm on all graphs on $m=5$ nodes. There exist more than $100,000$ such graphs and for many of them the Garcia-Puente algorithm does not finish within reasonable time. \\

\begin{table*}[t]\centering 
\begin{tabular}{@{}c|rrrrrr@{}}\toprule
\# Edges & Total & GP-ID & Time GP & DegBd-ID & Time DegBd  \\ 
\midrule
7 & 2 & 1 & 0.15 & 1 & 0.13 \\
9 & 6 & 4 & 0.71 & 5 & 0.19 \\
10 & 8 & 7 & 0.69 & 7 & 0.17 \\
11 & 22 & 13 & 0.95 & 13 & 0.17 \\
12 & 30 & 19 & 1.11 & 25 & 0.19 \\
13 & 43 & 19 & 1.48 & 31 & 0.22 \\
14 & 69 & 21 & 2.49 & 53 & 0.24 \\
15 & 82 & 24 & 1.25 & 65 & 0.27 \\
16 & 91 & 12 & 3.34 & 64 & 0.32 \\
17 & 106 & 5 & 1.27 & 75 & 0.41 \\
18 & 103 & 5 & 1.09 & 66 & 0.54 \\
19 & 91 & 2 & 3.19 & 63 & 0.72 \\
20 & 100 & - & - & 68 & 0.99 \\
21 & 71 & - & - & 46 & 1.53 \\
22 & 56 & - & - & 39 & 2.13 \\
23 & 43 & - & - & 30 & 1.74 \\
24 & 26 & - & - & 18 & 3.36 \\
25 & 18 & - & - & 10 & 3.81 \\
26 & 11 & - & - & 6 & 5.09 \\
27 & 10 & - & - & 8 & 3.97 \\
28 & 5 & - & - & 3 & 4.70 \\
29 & 4 & - & - & - & - \\
31 & 3 & - & - & 1 & 8.44 \\
\bottomrule
\end{tabular}
\caption{Counts and average computing time of $1000$ randomly sampled graphs on $10$ nodes. We compare the Garcia-Puente algorithm (GP) with Algorithm~\ref{alg:degree-bounded-identification} (DegBd).}
\label{table:randomly-generated}
\end{table*}

Since differences in computing time are greater on larger graphs, we consider a second experimental setup. We randomly generate $1000$ acyclic mixed graphs on $m=10$ nodes. Both directed and bidirected edges are sampled from an Erd\H{o}s--R\'enyi model with edge probability $0.2$. Again, we set the maximal degree in Algorithm~\ref{alg:degree-bounded-identification} to $d=5$. Since the Garcia-Puente algorithm does not finish on a large fraction of the sampled graphs within reasonable time, we set a maximal time limit for both algorithms to $10$ seconds. Said differently, we compare how many graphs each of the methods is able to identify if we allow a maximal computing time of $10$ seconds. Note that in our experience, either the Gröbner basis computations in the Garcia-Puente method finish immediately or almost ``never''. Table \ref{table:randomly-generated} reports the number of identified graphs and the average computational time needed to certify identifiability among the graphs that are rationally identifiable. While the GP-algorithm only certifies identifiability of $132$ graphs in total, Algorithm~\ref{alg:degree-bounded-identification} certifies $697$ graphs to be identifiable. Remarkably, for graphs with $20$ or more edges, the Garcia-Puente algorithm did not identify any graph while our method is still able to verify rational identifiability of many graphs within reasonable time. 

\section{Comparison to Existing Sufficient Criteria} \label{sec:comparison-to-existing-criteria}
In this section, we show that our algorithm subsumes many polynomial-time sufficient criteria in the literature under a bound on the maximal number of parents. This means that our algorithm certifies rational identifiability for a given input graph if $d$ is large enough whenever the given criteria certify it.  A caveat is that the criteria from the literature are polynomial time, while our algorithm is, in general, only quasi-polynomial time. On the other hand, our algorithm provably finds an identification formula up to degree $d$ whenever such a formula exists, while the criteria in the literature are only sufficient conditions. To show that Algorithm~\ref{alg:degree-bounded-identification} subsumes another criterion, observe that by Corollary~\ref{cor:degree} it suffices to show that the criterion yields an identifying order together with a set of identifying polynomials such that the maximum degree of the identifying polynomials is bounded.

\begin{assumption} \label{ass:bound-parents}
    There is a constant $s \geq 1$ such that $|\pa(v)| \leq s$ for all $v \in V$.
\end{assumption}

The Half-Trek Criterion by \citet[Theorem 1]{foygel2012halftrek} was a breakthrough in developing sufficient graphical criteria for rational identifiability that run in polynomial time. If a graph $G$ is certified to be rationally identifiable by recursively applying the Half-Trek Criterion, we say that $G$ is HTC-identifiable.

\begin{proposition} \label{prop:subsume-HTC}
    Let $G$ be HTC-identifiable and suppose that Assumption~\ref{ass:bound-parents} holds. Then there is an identifying order $\prec$ and a corresponding set of identifying polynomials $P_{\prec}$ such that $ \max_{g \in P_{\prec}}\allowbreak \deg(g) \allowbreak \leq 2s+1$.
\end{proposition}

Next, we consider the Instrumental Cutset criterion by \citep[Theorem 5.1]{kumor2019efficient}. It also runs in polynomial time but only applies to acyclic graphs. A mixed graph $G=(V,D,B)$ is acyclic if the induced directed graph $(V,D)$ is acyclic.  If a graph $G$ is certified to be rationally identifiable by recursively applying the Instrumental Cutset criterion, we say that $G$ is IC-identifiable.

\begin{proposition} \label{prop:subsume-IC}
Let an acyclic mixed graph $G$ be IC-identifiable and suppose that Assumption~\ref{ass:bound-parents} holds. Then there is an identifying order $\prec$ and a corresponding set of identifying polynomials $P_{\prec}$ such that $ \max_{g \in P_{\prec}} \deg(g) \leq 2s+1$.
\end{proposition}

Finally, we consider the Auxiliary Cutset criterion, that is, to the best of our knowledge, the current state-of-the art criterion for rational identification of acyclic mixed graphs that runs in polynomial time. If a graph $G$ is certified to be rationally identifiable by the Auxiliary Cutset identification algorithm \citep{kumor2020efficient}, we say that $G$ is AC-identifiable.  

\begin{proposition} \label{prop:subsume-AC}
Let an acyclic mixed graph $G$ be AC-identifiable and suppose that Assumption~\ref{ass:bound-parents} holds. Then there is an identifying order $\prec$ and a corresponding set of identifying polynomials $P_{\prec}$ such that $ \max_{g \in P_{\prec}} \deg(g) \leq s^4+2s^3+s^2$.
\end{proposition}

\section{Conclusion}
We have proposed an algorithm for deciding rational identifiability in linear structural equation models using computational algebra. The algorithm provides a sufficient criterion for identifiability and provably finds all identification formulas up to a prespecified degree in quasi-polynomial time relative to the size of the graph that specifies the model. In graphs that are either highly connected or very sparse, that is, when the maximal trek length grows rapidly or is bounded by a constant, our algorithm runs in polynomial time. Importantly,  even when not all causal effects are identifiable within the chosen degree bound, the algorithm still returns all available identification formulas for a subset of the parameters.  A key advantage is that these formulas can be applied directly  to downstream tasks such as estimation.  Moreover, we expect the underlying algebraic theory to extend naturally to settings that allow for feedback loops, in which the graph may contain cycles. We also note that while we focused on the application of Algorithm \ref{alg:degree-bounded-identification} to linear structural equation models, it can actually be applied to any algebraic statistical model for which there exists a positive weight vector $w$  such that the elimination ideal of the parametrization is homogeneous with respect to the grading. Examples include sparse factor analysis models~\citep{sturma2026matching}, Lyapunov models \citep{dettling2023identifiability}, and phylogenetic models \citep[Chapter 15]{sullivant2018algebraic}.

The most important problem that should be tackled in the future is understanding the growth of the maximal degree that appears in an identification formula with respect to the growth and the structure of the graph. Our complexity results would then directly yield complexity bounds for deciding rational identifiability.

\section*{Acknowledgments}
The project has received funding from the European Research Council (ERC) under the European Union’s Horizon 2020 research and innovation programme (grant agreement No 883818). Benjamin Hollering was partially supported by the Alexander von Humboldt Foundation. 

\bibliographystyle{apalike} 
\bibliography{literature}

\newpage
\begin{appendix}
\begin{center}
{ \LARGE\bf \noindent Appendix} 
\end{center}
\section{Proofs of Sections~\ref{sec:algebraic-setup} - \ref{sec:weighting-and-homogenization}}
\begin{proof}[Proof of Lemma~\ref{lem:chracterization-identifiability-2}]
    If $q$ is rationally identifiable, then by Lemma \ref{lem:chracterization-identifiability-1} the ideal $\mathcal{I}$ contains an element of the form $q a(\sigma)  - b(\sigma)$ with $a,b \in \mathbb{R}[\sigma] \subseteq \mathbb{R}[\theta_{\id}, \sigma]$ and $a \not\in \ker(\tau_G)=\mathcal{I} \cap \mathbb{R}[\sigma]$. Since $a$ is a polynomial only in $\sigma$, it follows that $a \not\in \mathcal{I} \cap \mathbb{R}[\theta_{\id}, \sigma]$. 

    Conversely, suppose that $\mathcal{I}$ contains an element of the form $q a(\theta_{\id}, \sigma) - b(\theta_{\id}, \sigma)$ with $a,b \in \mathbb{R}[\theta_{\id}, \sigma]$ and $a \not \in \mathcal{I} \cap \mathbb{R}[\theta_{\id}, \sigma]$, which implies that $q = \frac{b(\theta_{id}, \tau_G(\sigma))}{a(\theta_{id}, \tau_G(\sigma))}$. Since all variables $\bar{q} \in \theta_{\id}$ are rationally identifiable, there exist polynomials $\bar{a},\bar{b} \in \mathbb{R}[\sigma]$ such that $\bar{a} \not\in \ker(\tau_G)$ and $\tau_G(\bar{b})/ \tau_G(\bar{a}) = \bar{q}$. By plugging in $\tau_G(\bar{b})/ \tau_G(\bar{a})$ for each $\bar{q} \in \theta_{\id}$, it follows that 
    \begin{align} \label{eq:id-formula}
       q 
       = \frac{\tau_G(\widetilde{b}(\sigma))}{\tau_G(\widetilde{a}(\sigma))}
    \end{align}
    for polynomials $\widetilde{a}, \widetilde{b} \in \mathbb{R}[\sigma]$ that we obtain from clearing denominators and using that $\tau_G$ is a ring homomorphism. Equation~\eqref{eq:id-formula} implies that $q\widetilde{a} - \widetilde{b} \in \mathcal{I}$. Since $\tau_G(\widetilde{a})$ can not be the zero polynomial, we also have that $\widetilde{a} \not\in \ker(\tau_G)$. Hence, we conclude by Lemma~\ref{lem:chracterization-identifiability-1} that $q$ is rationally identifiable.
\end{proof}

\begin{proof}[Proof of Proposition~\ref{prop:id-with-gb}]
The statement directly follows from the proof of Algorithm 1 in \citet[Appendix, Section 8]{foygel2012halftrek}, also see \citet[Proposition 4]{garcia2010identifying}. We also refer to the proof of our Theorem~\ref{thm:id-with-gb}, which implies the statement.
\end{proof}

\begin{proof}[Proof of Proposition~\ref{prop:hom-equal}]
    By the definition of the trek weighting, each polynomial $\sigma_{uv}^{w(\sigma_{uv})}-\tau_G(\sigma_{uv})^{\wh}$ is $w$-homogeneous. Hence, the ideal $\bar{\mathcal{I}}$ is $w$-homogeneous with respect to the trek weighting $w$. For Claim (i), note that the generators $\sigma_{uv}-\tau_G(\sigma_{uv})$ form a Groebner basis $F$ of $\mathcal{I}$ with respect to any $w$-graded elimination order for $\sigma$ in  $\mathbb{R}[\theta,\sigma]$.  This is true since the leading monomial of $\sigma_{uv}-\tau_G(\sigma_{uv})$ is given by $\sigma_{uv}$ and therefore any pair of leading monomials of the generators are relatively prime, which means that all S-polynomials of the generators reduce to zero modulo $F$, see \citet[Chapter 2]{cox2015ideals}. We conclude that the $w$-homogenized polynomials $(\sigma_{uv}-\tau_G(\sigma_{uv}))^{\wh}=\sigma_{uv}-\tau_G(\sigma_{uv})^{\wh}$ are a basis of $\mathcal{I}^{\wh}$, see Theorem 4 in \citet[Section 8.4]{cox2015ideals}.  

    To show Claim (ii), consider a polynomial $f \in \mathcal{\mathcal{I}} \cap \mathbb{R}[\theta_{\id},\sigma]$. It follows that $f^{\wh} \in \mathcal{\mathcal{I}}^{\wh} \cap \mathbb{R}[\theta_{\id},\sigma,h]$, and thus we have shown the inclusion ``$\supseteq$''. On the other hand, if $f \in \mathcal{I}^{\wh} \cap \mathbb{R}[\theta_{\id}, \sigma, h]$, then $f^{\deh} \in (\mathcal{I}^{\wh} \cap \mathbb{R}[\theta_{\id}, \sigma, h])^{\deh}$ and hence $f^{\deh} \in (\mathcal{I}^{\wh})^{\deh} \cap \mathbb{R}[\theta_{\id}, \sigma] = \mathcal{I} \cap \mathbb{R}[\theta_{\id}, \sigma]$. But if $f^{\deh} \in \mathcal{I} \cap \mathbb{R}[\theta_{\id}, \sigma]$, then it must also be the case that $f \in (\mathcal{I} \cap \mathbb{R}[\theta_{\id}, \sigma])^{\wh}$.
\end{proof}

\begin{proof}[Proof of Lemma~\ref{lem:chracterization-identifiability-hom}]
First, note that by Proposition~\ref{prop:hom-equal} (ii), it holds that $\mathcal{I}^{\wh} \cap \mathbb{R}[\theta_{\id}, \sigma,h]= (\mathcal{I} \cap \mathbb{R}[\theta_{\id}, \sigma])^{\wh}$.  Now, suppose that $q \in \theta_{\rem}$ is rationally identifiable. Then, by Lemma~\ref{lem:chracterization-identifiability-2}, the ideal $\mathcal{I}$ contains an element of the form $q a(\theta_{\id}, \sigma) - b(\theta_{\id}, \sigma)$ with $a,b \in \mathbb{R}[\theta_{\id}, \sigma]$ and $a \not \in \mathcal{I} \cap \mathbb{R}[\theta_{\id}, \sigma]$. Hence, the ideal $\mathcal{I}^{\wh}$ contains the polynomial $(q a(\theta_{\id}, \sigma) - b(\theta_{\id}, \sigma))^{\wh} =q \tilde{a}(\theta_{\id}, \sigma,h) - \tilde{b}(\theta_{\id}, \sigma,h)$ with $\tilde{a},\tilde{b} \in \mathbb{R}[\theta_{\id}, \sigma,h]$. Since $\tilde{a}^{\deh}=a$ and $a \not \in \mathcal{I} \cap \mathbb{R}[\theta_{\id}, \sigma]$, it holds that $\tilde{a} \not \in (\mathcal{I} \cap \mathbb{R}[\theta_{\id}, \sigma])^{\wh}$. 

For the other direction, suppose that the ideal $\mathcal{I}^{\wh}$ contains an element of the form $q a(\theta_{\id}, \sigma,h) - b(\theta_{\id}, \sigma,h)$ with $a,b \in \mathbb{R}[\theta_{\id}, \sigma,h]$ and $a \not \in \mathcal{I}^{\wh} \cap \mathbb{R}[\theta_{\id}, \sigma,h]$. It follows that the dehomogenized polynomial $q a(\theta_{\id}, \sigma,h)^{\deh} - b(\theta_{\id}, \sigma,h)^{\deh}$ is an element of $\mathcal{I}$ and $a(\theta_{\id}, \sigma,h)^{\deh} \not\in \mathcal{I} \cap \mathbb{R}[\theta_{\id}, \sigma]$. We conclude by Lemma~\ref{lem:chracterization-identifiability-2} that $q$ is rationally identifiable.
\end{proof}

\begin{proof}[Proof of Theorem~\ref{thm:id-with-gb}]
Since $F$ is a reduced Gröbner basis of an homogeneous ideal, it consists of $w$-homogeneous polynomials.  Let $\prec_{\text{gr}}$ be a $w$-graded lex-elimination order for $\theta_{\rem}$. Consider the induced non-graded lex-elimination order $\prec$ for $\theta_{\rem}$. Then $F$ is also a reduced Gröbner basis with respect to this order since the leading terms and hence the S-polynomials coincide.  By the elimination theorem, it follows that $H = F \cap \mathbb{R}[\theta_{\id}, \sigma, h]$ is a reduced Gröbner basis of $\mathcal{I}^{\wh} \cap \mathbb{R}[\theta_{\id}, \sigma, h]$.  

Now, assume that $q$ is rationally identifiable, which implies by Lemma~\ref{lem:chracterization-identifiability-hom} that $\mathcal{I}^{\wh}$ contains a polynomial $f(q, \theta_{\id}, \sigma, h) = q a(\theta_{\id}, \sigma, h) - b(\theta_{\id}, \sigma, h)$ with $a,b \in \mathbb{R}[\theta_{\id}, \sigma, h]$ and $a \not\in \mathcal{I}^{\wh}\cap \mathbb{R}[\theta_{\id}, \sigma, h]$. We reduce $a$ by $H$ to get a remainder $\widetilde{a}$. Since $a \not\in \mathcal{I}^{\wh}\cap \mathbb{R}[\theta_{\id}, \sigma, h]$, the remainder $\widetilde{a}$ is nonzero, which also implies that $\widetilde{a} \not\in  \mathcal{I}^{\wh}\cap \mathbb{R}[\theta_{\id}, \sigma, h]$. Now, consider the modified polynomial $\widetilde{f}(q, \theta_{\id}, \sigma, h) = q \widetilde{a}(\theta_{\id}, \sigma, h) - b(\theta_{\id}, \sigma, h)$. Since $F$ is a Gröbner basis of $\mathcal{I}^{\wh}$, the leading term of $\widetilde{f}$ is divisible by some leading term of some polynomial in $F$. But since $\widetilde{a}$ is already reduced by $H$, the leading term of $\widetilde{f}$ is not divisible by any leading monomial of any element of $H$. Hence, it must be divisible by some element $g \in F \setminus H$ whose leading term has a nonzero degree in at least one of the variables in $\theta_{\rem}$. But in order to divide the leading term of $\widetilde{f}$, the leading term  $g$ must have degree one in $q$ and degree zero in all variables in $\theta_{\rem} \setminus \{q\}$. Thus, we can conclude that $g \in F$ is of the required form. 

Conversely, suppose that  $F$ contains a polynomial  with leading term $q a(\theta_{\id}, \sigma,h)$ for some $a \in \mathbb{R}[\theta_{\id}, \sigma, h]$. Since $F$ is a reduced Gröbner basis with respect to a lex-elimination order, this polynomial must be of the form
$q a(\theta_{\id}, \sigma,h) - b(\theta, \sigma,h)$ where $b$ only contains $\theta$-variables smaller that $\theta$ with respect to the lexicographic order on $\mathbb{R}[\theta_{\rem}]$. Moreover, $a$ does not reduce to zero by reduction of $H$, that is, $a \not\in \mathcal{I}^{\wh} \cap \mathbb{R}[\sigma, \theta_{\id},h]$. By Lemma~\ref{lem:chracterization-identifiability-hom}, it follows  that $q$ is rationally identifiable if all smaller variables in $\theta_{\rem}$ are rationally identifiable. Since we assume that $F$ contains such a polynomial for all $q \in \theta_{\rem}$, it follows that all $q \in \theta$ are rationally identifiable.
\end{proof}

\begin{proof}[Proof of Corollary~\ref{cor:degree}]
By Definition~\ref{def:identifying-polynomial}, the ideal $\mathcal{I}$ contains a polynomial of the form $q a(\theta_{\id}, \sigma) - b(\theta_{\id}, \sigma)$ with $a,b \in \mathbb{R}[\theta_{\id}, \sigma]$ and $a \not\in I \cap \mathbb{R}[\theta_{\id}, \sigma]$, which is of degree at most $\text{ID}_{G}$. By $w$-homogenizing, it follows as in the proof of Lemma~\ref{lem:chracterization-identifiability-hom} that $\mathcal{I}^{\wh}$ contains a polynomial $f(q, \theta_{\id}, \sigma, h) = q a_f(\theta_{\id}, \sigma, h) - b_f(\theta_{\id}, \sigma, h)$ with $a_f,b_f \in \mathbb{R}[\theta_{\id}, \sigma, h]$ and $a_f \not\in \mathcal{I}^{\wh}\cap \mathbb{R}[\theta_{\id}, \sigma, h]$. Moreover, note that $\deg_w(f)\leq \text{ID}_G \cdot w_{\trek}$ by the definition of the trek-weighting and the way $w$-homogenization is defined. By Theorem~\ref{thm:id-with-gb} and since $q$ is the smallest among the $\theta$-variables, we conclude that  the Groebner basis $F$  contains a polynomial $g_q := q a(\theta_{\id}, \sigma,h) - b(\theta_{\id}, \sigma, h)$. Since this polynomial has leading term that divides the leading term of $f$, and the Groebner basis $F$ is given with respect to a $w$-graded order, we conclude that it must also be the case that $\deg_w(g_q) \leq \text{ID}_G \cdot w_{\trek}$.
\end{proof}

\section{Proof of Theorem~\ref{thm:main-result}}

The proof of Theorem~\ref{thm:main-result} is based on three lemmas. The first is Lemma~\ref{lem:superfactorial} which is given in the main text. It relates the product of all trek weights to the superfactorial.

\begin{proof}[Proof of Lemma~\ref{lem:superfactorial}]

Consider a trek $\pi$ of length $w_{\trek}$ in $G$. It is either of the form
\[
    v_1\leftarrow \cdots \leftarrow v_k \leftrightarrow v_{k+1} \rightarrow \cdots  \rightarrow v_{w_{\trek}+1} ,
\]
or of the form
\[
    v_1\leftarrow \cdots \leftarrow v_k = v_{k+1} \rightarrow \cdots  \rightarrow v_{w_{\trek}+1}.
\]
By the definition of a trek,  the nodes $v_1, \ldots, v_k$ are pairwise distinct and the nodes $v_{k+1}, \ldots, v_{w_{\trek}+1}$ are pairwise distinct, but the sets $\{v_1, \ldots, v_k\}$ and $\{v_{k+1}, \ldots, v_{w_{\trek}+1}\}$ are allowed to intersect. However, any node $v \in V$ may at most appear twice on the trek $\pi$. \\
Now, consider any integer $i \in \{1, \ldots, w_{\trek}\}$. Note that there is a trek in $G$ from $v_s$ to $v_{s + w_{\trek}-i+1}$ for all $s=1, \ldots, i$. Hence, it holds that $w(\sigma_{v_s v_{s + w_{\trek}-i+1}}) \geq w_{\trek}-i+1$ for all $s=1, \ldots, i$. Moreover, for any two possibly equal nodes $u,v \in V$, it holds that there are at most four pairs of integers $\{l,t\} \subseteq \{1, \ldots, w_{\trek}+1\}$ such that $\sigma_{uv}=\sigma_{v_l v_t}$. We conclude that 
\begin{align*}
    \left(\prod_{u \leq v} w(\sigma_{uv})\right)^4
    &\geq \prod_{i=1}^{w_{\trek}} (w_{\trek}-i+1)^i \\
    &= w_{\trek}^1 \cdot (w_{\trek}-1)^2 \cdots 1^{w_{\trek}} \\
    &= w_{\trek}! \cdot (w_{\trek}-1)!  \cdots  1! \\
    &= \text{sf}(w_{\trek}).
\end{align*}
\end{proof}

The next lemma is about the asymptotic equivalence of the superfactorial.

\begin{lemma} \label{lem:superfactorial-approx}
    Let $n \in \mathbb{N}$ and consider the $n$-th superfactorial $\text{sf}(n) = \prod_{i=1}^n i!$. Then, for any non-decreasing function $f: \mathbb{N} \rightarrow \mathbb{N}$ with $\lim_{n \to \infty} f(n) = \infty$, we have
    \[
        \lim_{n \to \infty}\frac{\frac{1}{2}f(n)^2\log(f(n))}{\log(\text{sf}(f(n)))} = 1.
    \]
\end{lemma}
\begin{proof}
Since $f$ is non-decreasing and $\lim_{n \to \infty} f(n) = \infty$, it is enough to consider the identity function, i.e., $f(n)=n$ for all $n \in \mathbb{N}$. For $n \to \infty$, we obtain from the Stirling expansion of the Barnes $G$-function that
\begin{equation} \label{eq:barnes-g-expansion}
    \log(\text{sf}(n)) = \frac{1}{2}n^2 \log(n) - \frac{3}{4} n^2 + \frac{1}{2} n \log(2\pi) - \frac{1}{12}\log(n) + \frac{1}{12} - \log(A) + \mathcal{O}\left(\frac{1}{n}\right)
\end{equation}
see, for example, \citet[Appendix]{voros1987spectral}. Here, the constant $A$ is the Glaisher–Kinkelin constant \citep[Section 2.15]{finch2019mathematical}. The statement of the lemma now follows directly from Equation~\eqref{eq:barnes-g-expansion}.
\end{proof}

The last lemma is needed for proving the case $w_{\trek} \geq 2d \log(3d |V|)$ in Theorem~\ref{thm:main-result}.

\begin{lemma} \label{lem:helper}
    Let $n \in \mathbb{N}$ and consider the $n$-th superfactorial $\text{sf}(n) = \prod_{i=1}^n i!$. Let $d \geq 1$ and let $f: \mathbb{N} \rightarrow \mathbb{N}$ be a non-decreasing function such that $f(n) \geq \log(3dn) 2d$. Then
    \[
        \frac{(3 d n)^{2 d f(n)}}{\text{sf}(f(n))^{1/4}} = O(1).
    \]
\end{lemma}
\begin{proof}
In this proof, we let $C \geq 1$ be an absolute constant (not depending on $d$ and $n$) that might change its value from place to place. By Lemma~\ref{lem:superfactorial-approx}, we have $-\log(\text{sf}(f(n))) \leq -\frac{1}{2} C f(n)^2 \log(f(n))$. It follows that
\begin{align*}
    \log\left(\frac{(3 d n)^{2 d f(n)}}{\text{sf}(f(n))^{1/4}}\right) 
    &= 2 d f(n) \log(3dn) - \frac{1}{4} \log(\text{sf}(f(n))) \\
    &\leq C \left(2 d f(n) \log(3dn) - \frac{1}{8} f(n)^2 \log(f(n))\right) \\
    &\leq C \left(2 d f(n) \log(3dn) - \frac{1}{8} f(n) \log(f(n)) \log(3dn) 2d \right) \\
    &= C \left(2 d f(n) \log(3dn) (1 - \frac{1}{8} \log(f(n)) \right).
\end{align*}
For large $n$, the term $1 - \frac{1}{8} \log(f(n))$ becomes negative due to the assumption that $f(n) \geq \log(3dn) 2d$, that is, $\lim_{n \to \infty} f(n) = \infty$. Therefore, taking the exponential on both sides in the above inequality concludes the proof.
\end{proof}

\begin{proof}[Proof of Theorem~\ref{thm:main-result}]
If $G$ is rationally identifiable and $d \geq \text{ID}_G$, it follows directly from Corollary~\ref{cor:degree} that Algorithm~\ref{alg:degree-bounded-identification} returns ``yes''. Moreover, it returns ``no'' if $G$ is not rationally identifiable due to Theorem~\ref{thm:id-with-gb}. Hence, it is left to show the statement about the complexity. Note that the number of directed and bidirected edges in a graph is each smaller than $\binom{|V|}{2} \leq |V|^2$, and that the maximum length $w_{\trek}$ of a trek in an acyclic graph $G$ is bounded above by $2|V|-1 \leq 2|V|$. It follows that, in Algorithm~\ref{alg:degree-bounded-identification}, we compute a degree-bounded Groebner basis at most 
\[
    d \, w_{\trek} \, |\theta|^2 = d \, w_{\trek} \,  (|D|+|B|+|V|)^2=  O(d \, |V|^5)
\]
times. Therefore, we have by Proposition~\ref{prop:deg-bounded-gb-complexity} that the worst case-complexity is given by
\[
O \left(d |V|^5 \frac{1}{(\prod_{u \leq v} w(\sigma_{uv}))^\alpha} \binom{2|V|^2+d w_{\trek}-1}{d w_{\trek}}^\alpha \right),
\]
where we used that the numbers of variables in the ring $\mathbb{R}[\lambda, \omega, \sigma, h]$ is given by 
\begin{align*}
    |D|+|B|+|V| + \binom{|V|+1}{2} + 1 &\leq \binom{|V|}{2} + \binom{|V|}{2} + |V| + \binom{|V|+1}{2} + 1\\
    &= |V|^2 + \binom{|V|+1}{2} + 1 \\
    &\leq 2 |V|^2.
\end{align*}
 Now, observe first that
\begin{align*}
\binom{2 |V|^2+d w_{\trek}-1}{d w_{\trek}} &= O(\{2 |V|^2+d w_{\trek}-1\}^{d w_{\trek}}) \\
&= O(\{2 |V|^2+d (2|V|+1)-1\}^{d w_{\trek}}) \\
&= O(\{3d|V|^2\}^{d w_{\trek}}) = O(\{3d|V|\}^{2d w_{\trek}}),
\end{align*}
where we have used that $d \geq 2$ in the input to Algorithm~\ref{alg:degree-bounded-identification}. If $w_{\trek} \leq 2d \log(3d |V|)$, then the worst case complexity of Algorithm~\ref{alg:degree-bounded-identification} is clearly given by 
\[
    O\left( d |V|^5 \{3d|V|\}^{4\alpha d^2\log(3d|V|)}\right).
\]
Similarly, If $w_{\trek} \leq C$ for an absolute constant $C$, then the worst case complexity of Algorithm~\ref{alg:degree-bounded-identification} is  given by 
$
    O\left( d |V|^5 \{3d|V|\}^{2 \alpha d C}\right).
$
The remaining case is where $w_{\trek} \geq 2d \log(3d |V|)$. But in this case, Lemma~\ref{lem:superfactorial} and Lemma~\ref{lem:helper} imply that the complexity of Algorithm~\ref{alg:degree-bounded-identification} is bounded by
$
    O(d|V|^5).
$ 
We conclude the proof by noting that both $d |V|^5 \{3d|V|\}^{2\alpha d C}$ and $d|V|^5$ are of the order of $d |V|^5 \{3d|V|\}^{4\alpha d^2\log(3d|V|)}$.
\end{proof}

\section{Proofs of Section~\ref{sec:comparison-to-existing-criteria}}
\begin{proof}[Proof of Proposition~\ref{prop:subsume-HTC}]
If a graph is HTC-identifiable, then the identifying polynomials for the parameters $\lambda$ are given as follows. Let $\pa(v)=\{p_1, \ldots, p_k\}$ be the set of parents of a node $v \in V$. Moreover, let  $Y^1_v, Y^2_v \subseteq V$ be  two disjoint sets such that $|Y^1_v|+|Y^2_v|=k$ and let $Y^1_v \sqcup Y^2_v  = \{y_1, \ldots, y_k\}$ be their union. Define a matrix $A$ and a vector $b$ as 
\[
    A_{ij} = \begin{cases}
        [(I-\Lambda)^{\top} \Sigma]_{y_i p_j} & \text{ if } y_i \in Y^1_v, \\
        \Sigma_{y_i p_j} & \text{ if } y_i \in Y^2_v,
    \end{cases}
    \quad 
    \text{and}
    \quad
    b_i = \begin{cases}
        [(I-\Lambda)^{\top} \Sigma]_{y_i v} & \text{ if } y_i \in Y^1_v, \\
        \Sigma_{y_i v} & \text{ if } y_i \in Y^2_v.
    \end{cases}
\]
For each $j \in [k]$, we consider the polynomials $\lambda_{p_j v} \det(A) - \det(A_j)$, where $A_j$ is formed by replacing the $j$-th column of $A$ by $b$. HTC-identifiability implies that there exist sets $Y^1_v, Y^2_v \subseteq V$ of the above form for each node $v \in V$ such that the polynomials $\lambda_{p_j v} \det(A) - \det(A_j)$ are identifying polynomials for all $\lambda_{p_j v}$ with respect to an identifying order $\prec$. In particular, the order ensures that the set $Y^1_v$ only contains nodes such that the matrix
$A$ and the vector $b$ only contain variables $\lambda_{w y}$ that precede $\lambda_{p v}$ with respect to $\prec$.  Now, observe that, for each $v \in V$, the entries of $A$ and $b$ are polynomials of  degree at most $2$, which implies by Assumption~\ref{ass:bound-parents} that $\det(A)$ and $\det(A_j)$ are polynomials of degree at most $2 \pa(v) \leq 2 s$.

To conclude the proof, we place all parameters $\omega_{uv}$ after the parameters $\lambda$ in the identifying order. By Remark~\ref{rem:lambdas-enough}, this yields identifying polynomials for all parameters $\omega_{uv}$ of degree at most $3 \leq 2s+1$.
\end{proof}

\begin{proof}[Proof of Proposition~\ref{prop:subsume-IC}]
IC-identifiability implies that there is an identifying order $\prec$ on the parameters $\lambda$. Consider a variable $\lambda_{xy} \in \lambda$ and let $\theta_{\id} = \{\lambda_{ij} \in \lambda: \lambda_{ij} \prec \lambda_{xy}\}$ be the preceding variables that are already certified to be rationally identifiable by the IC criterion. Let $\Lambda^{\id}$ be the matrix with entries $\Lambda^{\id}_{ij} = \lambda_{ij}$ if $\lambda_{ij} \in \theta_{\id}$ and $\Lambda^{\id}_{ij} = 0$ else. Then, the identifying polynomial for $\lambda_{xy}$ is of the following form.

Let $T \subseteq \pa(y) \setminus \{x\}$ be a subset of the parents of $y$, and denote $T \sqcup \{x\}=\{t_1, \ldots, t_k\}$. Moreover, let $S_1, S_2 \subseteq V$ be two disjoint sets of nodes such that $|S_1| + |S_2| = k$, and let $S_1 \sqcup S_2 = \{s_1, \ldots, s_k\}$ be their union. Define a matrix $A$ and a vector $b$ as 
\[
    A_{ij} = \begin{cases}
        [(I-\Lambda^{\id})^{\top} \Sigma]_{s_i t_j} & \text{ if } s_i \in S_1, \\
        \Sigma_{s_i t_j} & \text{ if } s_i \in S_2,
    \end{cases}
    \quad 
    \text{and}
    \quad
    b_i = \begin{cases}
        [(I-\Lambda^{\id})^{\top} \Sigma(I-\Lambda^{\id})]_{s_i y} & \text{ if } s_i \in S_1, \\
        [\Sigma(I-\Lambda^{\id})]_{s_i y} & \text{ if } s_i \in S_2.
    \end{cases}
\]
Recall that $x = t_j$ for some $j \in [k]$. We consider the polynomial $\lambda_{xy} \det(A) - \det(A_j)$, where $A_j$ is formed by replacing the $j$-th column of $A$ by $b$. IC-identifiability implies that there exist sets $S_1,S_2$ and $T$  of the above form such that the polynomial $\lambda_{xy} \det(A) - \det(A_j)$ is an identifying polynomial for $\lambda_{xy}$ with respect to $\prec$. 

It remains to observe that the entries of $A$ are polynomials of  degree at most $2$, and that the entries of $b$ are polynomials of  degree at most $3$. By Assumption~\ref{ass:bound-parents} we conclude that $\det(A)$ is a polynomial of degree at most $2 \pa(y) \leq 2 s$ and that 
$\det(A_j)$ is a polynomial of degree at most $2 \pa(y) + 1 \leq 2 s + 1$. Hence, the degree of the identifying polynomial $\lambda_{xy} \det(A) - \det(A_j)$ is at most $2s+1$.

To conclude the proof, we place all parameters $\omega_{uv}$ after the parameters $\lambda$ in the identifying order. By Remark~\ref{rem:lambdas-enough}, this yields identifying polynomials for all parameters $\omega_{uv}$ of degree at most $3 \leq 2s+1$.
\end{proof}

\begin{proof}[Proof of Proposition~\ref{prop:subsume-AC}]
AC-identifiability implies that there is an identifying order $\prec$ on the variables $\lambda$. Consider a variable $\lambda_{xy} \in \lambda$ and let $\theta_{\id} = \{\lambda_{ij} \in \lambda: \lambda_{ij} \prec \lambda_{xy}\}$ be the preceding variables that are already certified to be rationally identifiable by the AC criterion. Let $\Lambda^{\id}$ be the matrix with entries $\Lambda^{\id}_{ij} = \lambda_{ij}$ if $\lambda_{ij} \in \theta_{\id}$ and $\Lambda^{\id}_{ij} = 0$ else. 

Now, we will construct the  identifying polynomial for $\lambda_{xy}$. Let $T \subseteq \pa(y) \setminus \{x\}$ be a subset of the parents of $y$, and let $S=\{s_1, \ldots, s_k\} \subseteq V$ be another set of nodes such that $|S|=|T|+1=k$. Define the $k \times k$ matrix
\[
    [(I-\Lambda^{\id}-D)^{\top}\Sigma]_{S,T \cup \{x\}}.
\]
Here, for two sets $C,Z \subseteq V$ such that $|C|=|Z|$, the $|V| \times |V|$ matrix $D$ is defined  as
\[
    D_{v w} = \begin{cases}
    \frac{\det(\Sigma_{Z, (C\setminus\{v\}) \cup \{w\}}^{\ast})}{\det(\Sigma_{Z, C})} & \text{ if } v \in C, \\
    0 & \text{ else},
    \end{cases}
\]
where $\Sigma_{Z, (C\setminus\{v\}) \cup\{w\}}^{\ast})$ is the matrix $\Sigma_{Z, C}$ with the column indexed by $v$ replaced by $[\Sigma(I-\Lambda^{\id})]_{Z,w}$.
We consider the polynomial
\begin{equation} \label{eq:id-polynomial-AC}
    \lambda_{xy} \det([(I-\Lambda^{\id}-D)^{\top}\Sigma]_{S,T \cup \{x\}}) - \det([(I-\Lambda^{\id}-D)^{\top}\Sigma]_{S,T \cup \{y\}}^{\ast}), 
\end{equation}
where $[(I-\Lambda^{\id}-D)^{\top}\Sigma]_{S,T \cup \{y\}}^{\ast}$ is the matrix $[(I-\Lambda^{\id}-D)^{\top}\Sigma]_{S,T \cup \{x\}}$ with the column indexed by $x$ replaced by $[(I-\Lambda^{\id}-D)^{\top} \Sigma(I-\Lambda^{\id})]_{S, y}$.

AC-identifiability implies that there exist sets $S,T,C$ and $Z$  of the above form with $|C|=|Z|\leq s$ such that, after clearing denominators in~\eqref{eq:id-polynomial-AC} by multiplying with $\det(\Sigma_{Z, C})^s$, the resulting polynomial $g_{xy}$ is an identifying polynomial for $\lambda_{xy}$ with respect to $\prec$. It remains to bound the degree of $g_{xy}$. The polynomial $\det(\Sigma_{Z, C})^s$ has at most degree $s^2$ and the polynomial $\det(\Sigma_{Z, C}^{-v, +w})$ has at most degree $s+1$. It follows that the degree of $g_{xy}$ is at most $s^2 (s (s+2) + 1)= s^4+2s^3+s^2$.

To conclude the proof, we place all parameters $\omega_{uv}$ after the parameters $\lambda$ in the identifying order. By Remark~\ref{rem:lambdas-enough}, this yields identifying polynomials for all parameters $\omega_{uv}$ of degree at most $3 \leq s^4+2s^3+s^2$.
\end{proof}

\end{appendix}
\end{document}